\documentclass[lettersize,journal]{IEEEtran}
\usepackage{amsmath,amsfonts}
\usepackage{algorithmic}
\usepackage{array}
\usepackage[caption=false,font=normalsize,labelfont=sf,textfont=sf]{subfig}
\usepackage{textcomp}
\usepackage{stfloats}
\usepackage{url}
\usepackage{verbatim}
\usepackage{graphicx}

\usepackage[ruled,linesnumbered,noend]{algorithm2e}
\usepackage{algorithmic}

\DontPrintSemicolon
\SetKwInOut{Input}{Input}
\SetKwInOut{Output}{Output}
\SetKw{KwAnd}{and}
\SetAlFnt{\small}

\usepackage{comment}
\usepackage{paralist}
\usepackage{color}
\usepackage{paralist}
\usepackage{xcolor}
\usepackage{amsmath}
\usepackage{amssymb}

\usepackage{amsthm}
\usepackage{float}

\usepackage{bm}
\usepackage{xspace}

\usepackage{enumitem}
\usepackage{pifont}
\usepackage[labelformat=simple, font=small, skip=3pt]{subcaption}
\usepackage[font=small, skip=3pt]{caption}
\usepackage[utf8]{inputenc}
\usepackage{csquotes}
\usepackage{graphicx}
\usepackage{soul,color}
\usepackage{lipsum}
\usepackage[noadjust]{cite}

\graphicspath{{figures/}}

\makeatletter
\newcommand*{\rom}[1]{\expandafter\@slowromancap\romannumeral #1@}
\makeatother

\floatstyle{boxed}
\newfloat{MILP}{tbp}{lop}

\newtheorem{theorem}{Theorem}

\theoremstyle{definition}
\newtheorem{definition}{Definition}

\setlist[description]{font=\normalfont\itshape\textbullet\space}

\begin{document}
\title{ Solving Stochastic Orienteering Problems with Chance Constraints Using Monte Carlo Tree Search } 
\author{
		Stefano Carpin~\IEEEmembership{Senior~Member,~IEEE}%
	\thanks{S. Carpin is with the Department of Computer Science and Engineering,   University of California, Merced, CA, USA. 
		This work is partially supported by the USDA-NIFA under award  \# 2021-67022-33452 and by the National Science Foundation (NSF) under Cooperative Agreement EEC-1941529. Any opinions, findings, conclusions, or recommendations expressed in this publication are those of the author and do not necessarily reflect the views of the U.S. Department of Agriculture or the NSF. 
	}
}

\maketitle

\begin{abstract}
We present a new Monte Carlo Tree Search (MCTS) algorithm to solve the stochastic
orienteering problem with chance constraints, i.e., a version of the problem
where travel costs are random, and one is assigned a bound on the tolerable probability of exceeding the
budget. The algorithm we present is online and anytime, i.e., it alternates planning and execution,
and the quality of the solution it produces increases as the allowed computational time
increases. Differently from most former MCTS algorithms, 
for each action available in a state the algorithm maintains estimates of both
its value and the probability that its execution will eventually result in a violation of the chance 
constraint. Then, at action selection time, our 
proposed solution prunes away trajectories that are estimated to violate the 
failure probability. Extensive simulation results show that this approach can quickly 
produce high-quality solutions and is competitive with the optimal but time-consuming solution.
\end{abstract}

\begin{ntp}
In many practical scenarios one is faced with multiobjective sequential decision making 
problems that can be solved through constrained optimization. If some of the parameters
are known with uncertainty, the event ``violating one of the constraints'' becomes a random
variable whose probability should be bound. As an application of this general problem
formulation, in this paper we consider stochastic orienteering, a problem that finds
applications when a robot is tasked with performing multiple tasks of varying utility while being
subject to a bound on the traveled distance. Many problems in logistics, precision agriculture, 
and environmental monitoring, just to name a few, can be cast as instances of this optimization
problem. 
\end{ntp}

\IEEEpeerreviewmaketitle	

\section{Introduction} 
Orienteering is a combinatorial optimization problem that can be used to model numerous
problems relevant to robotics and automation, 
such as logistics \cite{CarpinTRO2022}, environmental monitoring \cite{Rus2016}, 
surveillance \cite{Jorgensen2017}, and precision agriculture \cite{CarpinThayerTASE2020}, just to name
a few (see Figure \ref{fig:orienteering}).
\begin{figure}[htb]
	\centering

	\includegraphics[width=\columnwidth]{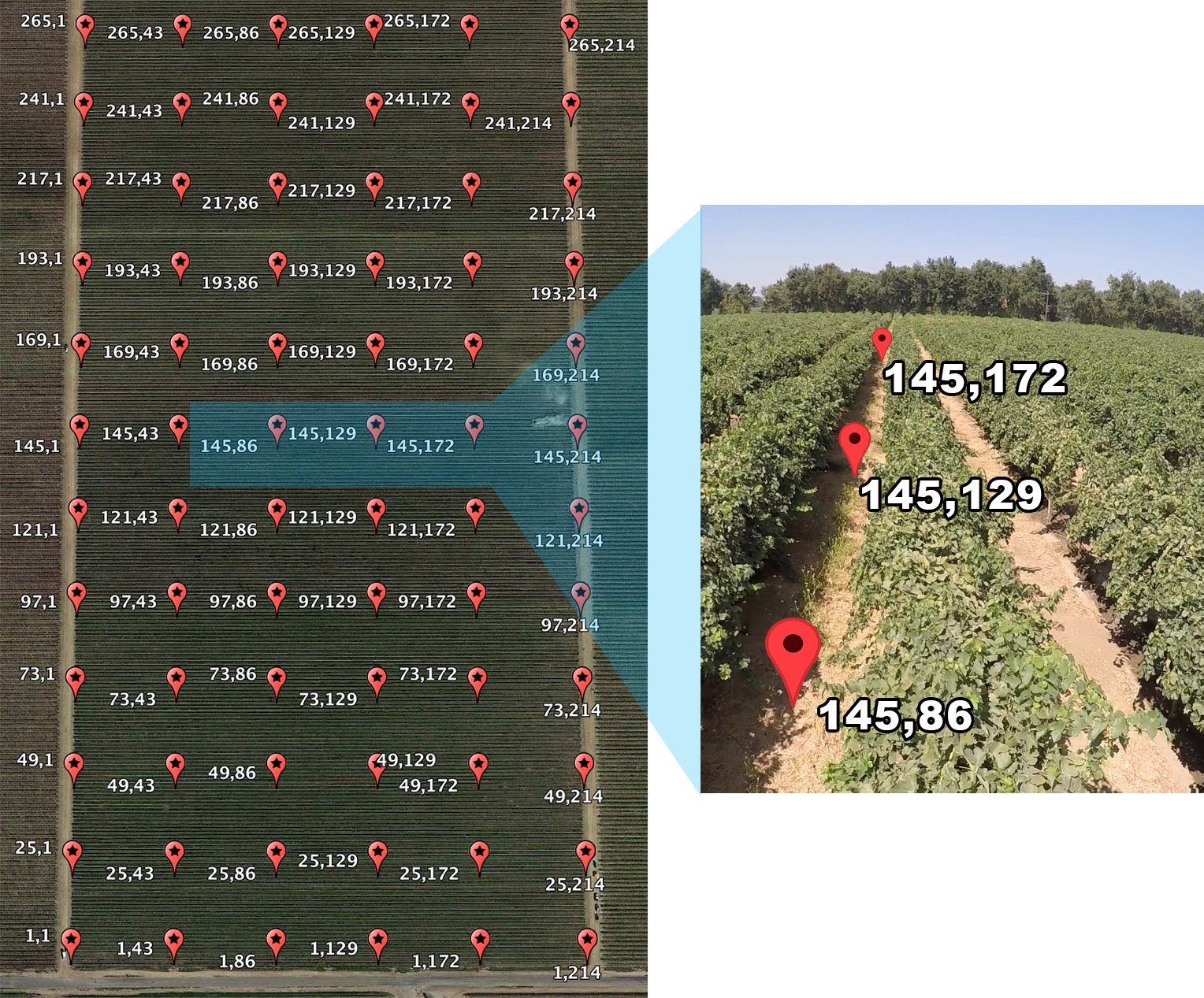}
	
	\caption{
		An example of the application of orienteering in viticulture.  The left panel shows an
		aerial view of a commercial vineyard located in California with more than 55,000 vines.
		Pins are displayed at locations where soil moisture samples must be collected.
		The left panel shows a zoomed  version of the vineyard, illustrating how the pins are 
		placed in the traversable regions between tree rows. 
		Due to the extension of the vineyard, a robot
		tasked with automatically collecting samples cannot visit all locations, and has to select a suitable
		subset. Not all sampling locations have the same value, and therefore this is a natural instance of the
		orienteering problem. Note that in this specific problem setting the robot cannot move
		along straight lines when moving between different locations, but rather has to exit either end
		of the vineyard before moving to a different tree row because of the irrigation infrastructure.}
	\label{fig:orienteering}
\end{figure}
 In its basic formulation, one is given a graph $G$ with rewards associated with vertices and costs
assigned to edges, as well as a budget $B$. The goal is to find a path  collecting the highest
sum of rewards of the visited vertices while ensuring that the path cost does not exceed the budget $B$.
In this paper, for the sake of simplicity, we assume that $B$ is a temporal budget, but it could as well be energy or any other
resource consumed by the robot as it moves from location to location.
Numerous variants have been proposed and, as discussed in section \ref{sec:sota}, most are 
computationally intractable. In robotics applications this model is usually adopted in
scenarios where vertices are associated with places that a robot must visit (e.g., to pick a
package, or to deploy a sensor, collect a sample, etc.), and edges are associated with paths connecting the different 
locations, with the edge cost modeling the time or energy spent by the robot to traverse it. 
Most former research in this area has considered scenarios where costs and rewards
are deterministic, and only a few works have explicitly considered cases where these are affected 
by uncertainty.  
In most practical applications, however, the time or energy spent to traverse an edge 
(i.e., to move from one location to  another) is not known upfront, but is rather a continuous 
random variable whose realization will  only be known at run time. 
When this is the case, it follows that for a given traversed path the incurred cost is also a random variable.
For example, the robot
may have to wait to traverse an aisle in a warehouse to give way to another robot coming in the
opposite direction, or it may have to take a detour because a passage is blocked, and so on.
The variant of the problem where edge costs are stochastic is known as the \emph{stochastic
orienteering problem}, and while studied in the past \cite{Campbell2011}, it has received far less attention than
its deterministic version. Obviously, the stochastic version is not simpler than the deterministic one.
Assuming that the probability density function (pdf) characterizing the traversal time of edges 
is known, some former solutions reduced the stochastic problem to the deterministic one by using the
expected traversal time. Such solutions are in many situations unsatisfactory, because optimizing
for expectation may lead to  realizations where the robot exceeds the budget $B$ while 
executing the task. This may be a major inconvenience because if the robot runs out of energy
and stops, it has to be recovered and recharged.

In our recent works \cite{CarpinRAL2021b,CarpinIROS2021b,CarpinIROS2021a} we  presented a new approach to solve the stochastic
orienteering problem whereby we introduce \emph{chance constraints}, i.e., while solving the
problem we consider a constraint on the probability that the stochastic cost of the path exceeds the 
assigned budget. It shall be mentioned that besides our works, there is very little former literature
attacking this specific problem, as discussed in Section \ref{sec:sota}.
Differently from all previous works, our recently developed algorithms produce a \emph{path policy}, i.e., a time-parametrized schedule that, 
depending on how much budget is left, determines where to move next while ensuring that the probability
of exceeding the budget remains bounded below an assigned constant. 
This approach is therefore adaptive, i.e., rather than producing one path, it gives a policy that will result in different paths depending
on the actual realizations of the edge traversal times.
This is achieved by reducing the 
planning problem to a constrained Markov Decision Process (CMDP) with a suitable structure. This approach,
while effective, has some limitations. First, the path policy is determined using an  initial solution to 
the deterministic version of the orienteering problem. This initial solution is computed using a heuristic, and if 
the heuristic selects a poor path, the algorithm has no way to move away from it. Second, in constructing
the finite CMDP, the continuous temporal dimension capturing the time left is discretized into time intervals.
This generates a tradeoff, where a finer grain discretization generates a  larger state space for
the associated CMDP, and this translates to increased computation time. Moreover, the discretization
naturally leads to an approximation of the underlying transition probabilities in the state space.

In an effort to overcome these limitations of  former solutions, in  this paper we present a completely different 
approach to solve the stochastic orienteering problem with chance constraints. The main idea is to repeatedly 
search the space of possible paths using a Monte Carlo Tree Search (MCTS) approach in an
online fashion.  Following a typical rolling horizon approach, after an initial path has been identified, only the 
first segment is executed. Then, based on the \emph{actual} time spent to execute the first motion, 
the remaining budget is updated, and the method run again. Additionally, to account for the chance constraint,
we introduce a novel \emph{backup} procedure based on Monte-Carlo sampling that 
eliminates from the search tree the paths that would violate the budget constraint.
Accordingly, the search in the tree is informed by a novel criterion we dub UCTF (Upper-bound Confidence
for Trees with Failures) that extends the widely used UCT formulation.
The original contributions of this paper are the following:
\begin{itemize}
\item we formulate the stochastic orienteering problem with chance constraints  as an MCTS planning problem;
\item we introduce novel tree policies and backup policies to incorporate and manage the probability of violating the given constraint;
\item we demonstrate that this approach is competitive with optimal, time-consuming solutions based on mixed integer linear programming.
\end{itemize}

The rest of the paper is organized as follows. Section \ref{sec:sota} discusses selected
related work, while detailed background on stochastic orienteering and MCTS algorithms
is provided in section \ref{sec:background}. Our new algorithm is introduced in 
section \ref{sec:algo} and theoretically analyzed in section \ref{sec:thprop}.
Experimental evaluations are then given in section \ref{sec:results} and we then draw conclusions
and discuss  future work in section \ref{sec:conclusions}.
\section{Related Work} 
\label{sec:sota}

The deterministic  orienteering problem was first formalized in \cite{Golden1987} where it was	also shown to 
be  $\mathit{NP}$-hard. Consequently, most solutions to the problem rely on heuristic approaches
\cite{CarpinThayerICRA2018,Gunawan2016}. Exact solutions for limited size instances
can be found  using integer programming formulations \cite{fischetti1998solving},
while approximate solutions have been proposed but have seen limited use \cite{Chekuri2012}.
Stochastic variants of the problem can encompass stochastic costs for the edges or stochastic
rewards for the vertices. As  the deterministic  orienteering problem is a
special case of stochastic orienteering, it follows that stochastic orienteering is $\mathit{NP}$-hard, too.
In \cite{Campbell2011} the authors propose  an exact solution for a special class of graphs, and various
heuristics for general graphs, but do not consider chance constraints, i.e., bounds on the 
probability of exceeding the budget. The works presented in \cite{Varakantham2018} 
and \cite{10.1007/978-3-642-41575-3_30} tackle a problem
similar to ours, inasmuch as they consider a risk-sensitive formulation for the stochastic orienteering problem.
In \cite{Varakantham2018}  the authors propose an algorithm to solve it based on local search,
while in \cite{10.1007/978-3-642-41575-3_30} the authors propose a mixed integer program based on sample average approximation.
These solutions are fundamentally different from the one we propose because they are
formulated offline a priori, and not updated as the mission unfolds based on the travel costs experienced during the mission.
This may lead to excessively conservative solutions that collect less reward on average (think for
example to the case when one follows a predetermined path where traversal times are much lower than 
expected.)
Our previous works \cite{CarpinRAL2021b,CarpinIROS2021b,CarpinIROS2021a} are the first ones to 
propose the concept of a \emph{path policy} while solving the stochastic orienteering problem with chance constraints.
While the computation is still offline, rather than computing a single path, these methods produce a set of rules
that can be queried at runtime to determine which vertex to visit next based on the remaining budget.
Therefore these solutions will return different paths depending on the specific realization of the stochastic processes
governing travel times along the edges.

Monte Carlo Tree Search (MCTS) encompasses a family of any-time methods to solve 
planning problems using generative models. Albeit a mature technique \cite{UCB,Coulom2007,MCTSSurvey},
MCTS gained significant popularity while being recently used in combination with reinforcement
learning, most notably in \cite{Mnih2015}. The use of MCTS algorithms for problems with chance
constraints has been so far limited. In \cite{ayton2018vulcan} the authors  propose an algorithm
for chance constrained Markov Decision Processes that is guaranteed to return a policy 
satisfying the chance constraint.  However, the  algorithm described in \cite{ayton2018vulcan}
assumes a discrete state space, because all explored state histories must be explicitly stored in a search tree 
to avoid repeated evaluations of already considered histories. This is evident in the case studies presented in \cite{ayton2018vulcan} to illustrate
its performance. The algorithm and application we present in this paper, instead, can handle continuous components in
the state space. In particular, in the stochastic orienteering problem 
 the state includes the residual budget, which is a continuous variable.

\section{Background}  
\label{sec:background}

In this section, we first summarize the general definition of chance constrained optimization,
and then show how to cast the stochastic orienteering problem with chance constraints (SOPCC in the following)
 as an instance of this class of problems. 
We also present some ideas and results related to the Sample Average Approximation approach (SAA)
that are used in our solution, as well as in the alternative method we compare our solution with.
In the first part we mostly follow the notation and formulation
introduced in \cite{SAApaper}.

\subsection{Chance Constrained Problems}
A chance constrained problem is an optimization problem formulated as
\begin{align}
&\min_{x \in X} f(x)  \label{eq:ccproblem} \\
&\textrm{s.t.}~ \Pr[L(x,\xi)  \leq 0] \geq 1-\alpha \nonumber 
\end{align}
where: $X$ is the space of possible solutions; $f: X\rightarrow \mathbb{R}$ is a real valued objective  function; 
$\xi$ is a random vector with known probability density function and support $\Xi \subset \mathbb{R}^m$; $\alpha \in (0,1)$ is an assigned confidence level; $L: X\times \Xi \rightarrow \mathbb{R}$. In general  $L$ could be a vector function, 
but we here consider the case where it is scalar as it is more directly connected to the problem we present later.
The constraint involving $L$ is the so-called \emph{chance constraint} because it imposes  an upper bound on the probability
of violating the inequality $L(x,\xi)  \leq 0$.
Chance constrained problems  are extensively used in portfolio optimization \cite{Rockafellar2000}, but have also been used to model
and solve robotics problems \cite{Ono2015,PavoneECC2020,ayton2018vulcan}.
It is well known that these problems are in general difficult to solve \cite{SAApaper,Shapiro2005}  and therefore
approximate solutions have been proposed. One such method is the SAA approach where one uses $N$ independently
identically distributed samples of $\xi$ to approximate the chance constraint. More precisely, problem~\eqref{eq:ccproblem}
can be equivalently rewritten as 
\[
\min_{x \in X}f(x) \qquad \textrm{s.t.} ~ ~ p(x) \leq \alpha
\]
where we define the function $p(x)$ as $p(x) = \Pr[L(x,\xi)>0],$
and we now evidence that $\alpha$ bounds the probability of violating the constraint 
	$L(x,\xi)  \leq 0$ introduced in \eqref{eq:ccproblem}.
Then, after having drawn $N$  independent samples  $\xi_1\dots \xi_N$ for the random vector $\xi$, one can approximate $p(x)$ with  the function 
\begin{equation}
\hat{p}_N(x) = \frac{1}{N}\sum_{i=1}^N \mathbb{I}(L(x,\xi_i))
\label{eq:saa}
\end{equation}
where $\mathbb{I}(\cdot)$ is the indicator function that is equal to 1 if its argument is positive and 0 otherwise.
As shown in \cite{SAApaper}, 
\[
\lim_{N\rightarrow \infty} \hat{p}_N(x)  = p(x).
\]

\subsection{Stochastic Orienteering with Chance Constraints}
The deterministic orienteering problem is defined as follows. Let $G=(V,E)$ be a directed graph with $n$ vertices, 
$r:V\rightarrow \mathbb{R}^+$ be a positive reward function defined over the set of vertices and $c:E\rightarrow \mathbb{R}^+$
be a positive cost function defined over the edges. Let $v_s\in V$ and $v_g \in V$ be assigned start and goal vertices, respectively.
In the following, without loss of generality, we assume $G$ is a complete graph. When this is not the case one can
simply add all missing edges and set their costs equal to the sum of the costs along the shortest path.
 For a path $\mathcal{P}$ over $G$ connecting $v_s$ to $v_g$, the reward of the path
$R(\mathcal{P})$ is defined as the sum of  rewards of the vertices along the path, with the stipulation that if  a vertex
is visited more than once, then its reward is collected just once. The cost of the path $C(\mathcal{P})$ is
instead the sum of the costs of the edges along $\mathcal{P}$, but in this case, if an edge appears multiple
times, its corresponding cost is charged every time. However, since we assumed that $V$ is a
complete graph, it is therefore evident that an optimal solution will never have to visit the same vertex twice
or traverse the same edge twice. For a given budget $B>0$, the orienteering problem
asks to solve the following constrained optimization problem 
\[
\mathcal{P}^* = \arg  \max_{\mathcal{P}\in \Pi} R(\mathcal{P}) \qquad \textrm{s.t. } C(\mathcal{P}) \leq B
\]
where $\Pi$ is the set of all possible paths  connecting $v_s$ with $v_g$. 
In this problem the space of (feasible or infeasible) solutions $\Pi$ is a discrete set consisting of
$(n-2)!$ paths from $v_s$ to $v_g$.\\
In the stochastic version of the problem, 
the cost of each edge $(v_i,v_j)$ is not a constant, but rather a continuous random variable $\xi_{i,j}$ with
a known probability  density function (pdf) with positive support and  finite expectation. 
The availability of these density functions is essential in our method, 
because we will use an approach similar to SAA  to estimate the probability that
a candidate solution violates the budget constraint.
 This assumption is
also consistent with MCTS literature where a generative model is assumed to be fully known to implement
the rollout step described later on.
 Note that for the viability of the method we propose,
it is not necessary for these edge pdfs to be independent. All that is needed is being able to draw  independent identically distributed samples
from these pdfs, as per the SAA method described by Eq.~\eqref{eq:saa}. This assumption also matches the process
 described above to turn an incomplete graph into a complete graph by adding missing edges.
  Edges added to the graph have a random cost equal to
 the sum of the random costs of the edges of shortest the path associated with the edge being added. The pdfs of these
 random costs will be in general not
 independent, but as just pointed out this dependency does not invalidate the applicability of the method we propose.\footnote{ For completeness, 
 one should add that the {\em shortest path} in this case is determined considering the expectations of the edges' random costs, which
we assumed to be finite.}
In the following, for the stochastic
version of the problem $c(v_i,v_j)$ is the expectation of the random variable $\xi_{i,j}$ associated with the edge 
$(v_i,v_j)$.
 In this case, for a given path $\mathcal{P} \in \Pi$ the corresponding path cost
$C(\mathcal{P})$ is therefore  a random variable given by the sum of the random variables
associated with the edges appearing along the path that will be written as $C(\mathcal{P},\xi)$,
following the notation introduced in the previous subsection.
Given a fixed failure probability $P_f$, the
SOPCC  asks to solve the following constrained optimization problem:
\begin{align}
&  \max_{\mathcal{P}\in \Pi} R(\mathcal{P}) \label{eq:sopcc} \\
& \textrm{s.t. } \Pr[C(\mathcal{P},\xi) > B]\leq P_f \nonumber
\end{align}
i.e., we now constrain the probability that the cost of the path exceeds the budget $B$.
Problem \eqref{eq:sopcc} is evidently an instance of the equivalent formulation of problem \eqref{eq:ccproblem}
discussed above.\footnote{While \eqref{eq:ccproblem} is formulated as a minimization problem, \eqref{eq:sopcc}
is formulated as a maximization problem for consistency with the orienteering literature, but the two solutions
are obviously equivalent. } \\
 \emph{Remark}: in the  definition of the stochastic orienteering problem with chance constraints $P_f$ 
	is a parameter that must be set according
	to the risk aversion of the end user, i.e., it represents the accepted probability of violating the  assigned
	budget constraint.  In the orienteering scenarios considered in the introduction, 
	this is the probability that the robot runs out of energy before being able to reach the goal vertex
	$v_g$ where batteries can be recharged or swapped. Risk averse users will prefer low values for
	$P_f$, while risk tolerant users may be willing to consider higher values, in return for larger rewards. The   $P_f$ value will in turn
	influence $\Pi$, i.e., the set of possible solution paths, with smaller values corresponding to smaller sets.  
	In section \ref{sec:results} we will solve the same problem instances for different values of $P_f$ and this will
	illustrate how rewards and computational time vary. In general,  it shall be noted that   the  assumption that the pdfs
	for the edge costs have finite 	expectations does not allow to set $P_f=0$. This is for example the case when 
	edge costs are distributed  according to exponential distributions, as we do in section \ref{sec:results}. 
	In this instance,  one cannot assume an upper bound on the edge  travel cost, and therefore setting $P_f=0$ 
	will yield an empty solution set $\Pi$ making the problem unfeasible. These aspects will be further expanded
	in section \ref{sec:results}.

In practical instances, one is of course not only interested in the value of the 
objective function $\mathcal{R}(\mathcal{P})$ but also in the optimal path $\mathcal{P}$.
When  solving this problem, one can compute the solution offline, i.e., before its execution starts,
or online, i.e., it may update the path based on the available residual budget. 
In \cite{CarpinRAL2021b,CarpinIROS2021b,CarpinIROS2021a} we proposed an online algorithm where we computed
 a \emph{path-policy} that  guides the robot through the vertices while being aware of the 
remaining budget. 
In these works the policy is computed off-line but it includes a family of paths, and the selection 
of the path to follow is performed online based on the realizations of the random variables associated with the 
already traversed edges.
The method presented in  \cite{10.1007/978-3-642-41575-3_30} is 
instead a fully offline solution where the path $\mathcal{P}$ is computed beforehand and then executed
irrespective of the realizations of the stochastic travel costs.
Alternatively, one can opt for a fully online approach, where the path is continuously refined
based on the time spent by the agent while traversing the edges. This can be implemented with the
MCTS approach discussed next.  

\subsection{Monte Carlo Tree Search}
MCTS is an approach to solving decision making problems in an online fashion, where planning and execution
alternate.  It belongs to the family of receding-horizon (also called rollout) methods \cite{BertsekasRollout}, whereby one solves the 
planning problem using a finite time-horizon, but then executes just the first action in the plan, and then 
re-plans from scratch based on the outcome of the first action. In MCTS planning, the algorithm builds
a rooted tree whose root node represents the current state, and whose edges connect states that can be reached
through the execution of a single action.
Key  elements in MCTS algorithms are the following: 
\begin{enumerate}
	\item a selection process to  move from the root of
	the tree down to a leaf following a so-called \emph{tree policy};
	\item an \emph{expansion} step executed to add leaf nodes to the tree; 
	\item  a \emph{rollout policy} to be executed from a leaf to estimate how ``good'' a leaf is;
	\item  a \emph{backup policy}
	to be executed from the leaves back to the root to guide the eventual selection of the best action from the root.
\end{enumerate}
After the tree is built, an action is selected among those available in the root node. The action 
is executed, and the tree is discarded and rebuilt having as root the vertex reached executing the selected action.
The reader is referred to \cite{SuttonRL} (chapter 8 and references therein) for more details.
As pointed out in \cite{pmlr-v48-khandelwal16}, different tree policies (step 1) and backup policies (step 4) 
may have a dramatic impact on the performance of MCTS based planning. In particular, while
UCT  (Universal Confidence bound for Trees) \cite{UCB} is often considered the standard tree policy, it is not directly 
applicable to our problem because actions yielding high values (adding a high reward for the path in the orienteering
problem) may also increase the probability of violating the chance constraint $ \Pr[C(\mathcal{P},\xi) > B]\leq P_f$
if they  have high edge costs.
For this reason, we propose an alternative tree policy based on UCT, but factoring in also failure probabilities
(we call this tree policy UCTF -- UCT with Failures).
Similarly, backup strategies based
on plain Monte Carlo averaging are not applicable because they do not consider whether constraints are
violated or not.	
Inspired by the complex backup strategies studied in \cite{pmlr-v48-khandelwal16}, in this work we instead propose
a backup policy that explicitly considers failure constraints.

\section{An Online MCTS Algorithm for Stochastic Orienteering with Chance Constraints} 
\label{sec:algo}

\subsection{Using trees to represent orienteering solutions}

The MCTS approach sketched in the previous section can be used to model and solve an instance of the SOPCC problem
as follows.
To each vertex $v \in V$ we associate a set of actions, i.e., the set of vertices that can be directly reached from $v$. 
Hence, each action corresponds to an edge in $G$.
The tree $\mathcal{T}$ 
is rooted at the vertex where the robot is currently positioned, and is parametrized by the available budget $B$.
Therefore, all quantities stored in $\mathcal{T}$ are relative to the available budget $B$. 
As the robot moves from location to location, the available budget $B$ is decreased to account for the already incurred travel costs.
There is a one-to-one correspondence
between nodes in the tree $\mathcal{T}$ and vertices in the graph. Each node in the tree may have from 0 to $n-1$ children. 
Vertex $v_j$ can be a child of $v_i$ in the tree only if there is an edge connecting $v_i$ with $v_j$, i.e., if there is
an action from $v_i$ that leads to $v_j$.  
Throughout the remainder of this section, when considering a tree $\mathcal{T}$ we ignore nodes associated with vertices
in $G$ that have already been visited by the robot, because they cannot be part of an optimal solution  as there is never value in 
revisiting an already visited vertex.
Figure \ref{fig:graphandtree} shows a tree associated with a simple graph
with five vertices.

\begin{figure}[h]
	\centering
	\includegraphics[width=0.8\linewidth]{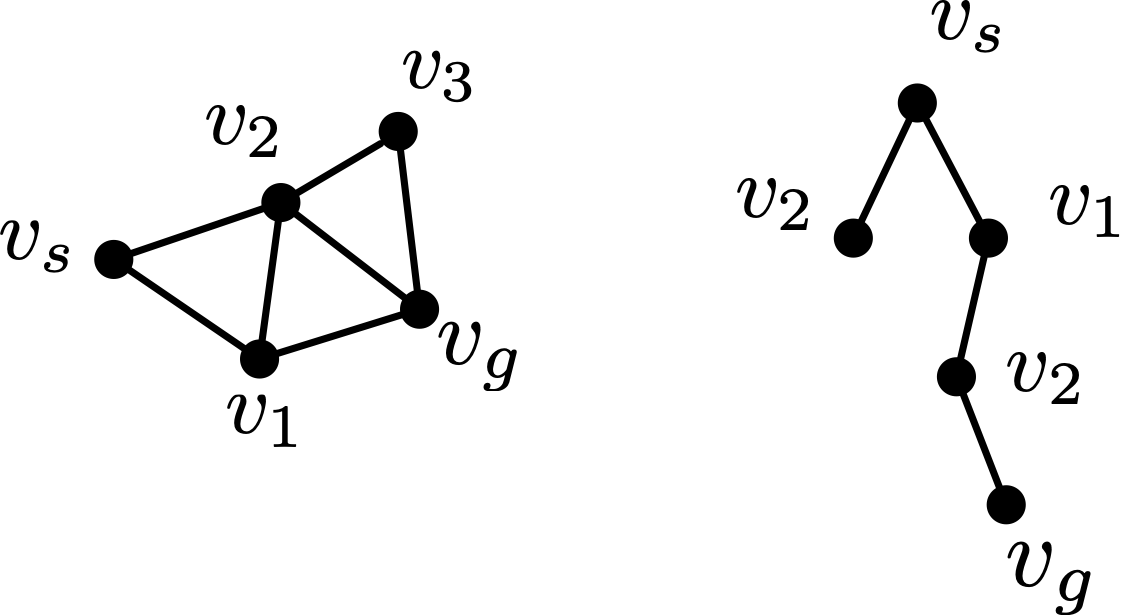}
	\caption{The right side of the figure shows a possible MCTS tree  $\mathcal{T}$ associated with the simple graph on the left
		and rooted in $v_s$ (start vertex).
		Vertices $v_1$ and $v_2$ are children of $v_s$ because they are directly connected to it.  Executing action $v_1$ from
		$v_s$ means moving from $v_s$ to $v_1$. Vertex $v_3$, not appearing in  the tree, cannot be a child of
		$v_s$ because it is not directly connected to it. Vertex $v_2$ appears as a child of both $v_s$ and $v_1$ because
		it is connected to both, but it occurs along two different paths starting from the root node $v_s$.  
		All paths in $\mathcal{T}$ from $v_s$ to a leaf encode  possible paths in $G$. In this simple
		example there are two paths, namely $v_s,v_2$, and $v_s,v_1,v_2,v_g$.  Note that while the MCTS is being built 
		not all paths must end at the goal vertex $v_g$.}
	\label{fig:graphandtree}
\end{figure}

For a given path from the root to a leaf, one can simulate the time it takes to traverse it
by adding random samples drawn from the known pdfs associated with the edges along the path.
For every internal tree node $v_i$ we store three attributes for each of its children $v_j$:
\begin{itemize}
	\item $N[v_j]$ is the  number of times that  action $v_j$ was attempted from $v_i$;
	\item $Q[v_j]$ is the expected reward associated with the \emph{feasible} path of maximum 
	reward   that selects $v_j$ from $v_i$, if it exists.
	Feasible, in this context, means that the estimated failure probability does not exceed the assigned bound $P_f$
	(see also Definition \ref{def:feasible} below.)
	If all paths connecting $v_i$ to $v_j$ violate the failure probability $P_f$, then $Q[v_j]$ is set to the expected
	reward associated with the paths starting from $v_i$ and going through $v_j$.
	\item $F[v_j]$ is the estimated failure probability  of the path defining the value $Q[v_j]$ just defined.
	Therefore, when the path is feasible $F[v_j]  \leq P_f$.
\end{itemize}
The first two attributed $N$ and $Q$ are borrowed from the classic MCTS definition, while $F$ is specifically 
introduced to solve the problem we study in this paper.
These three quantities are incrementally updated as the tree is being built and expanded, and the specifics will be given
when discussing the backup policy. With reference to the tree depicted in Figure \ref{fig:graphandtree}, the
root note $v_s$, stores the quantities $Q[v_1], Q[v_2],F[v_1],F[v_2],N[v_1]$ and $N[v_2]$
because $v_1$ and $v_2$ are connected  to $v_s$. Similarly $v_1$ stores its own 
 $Q[v_2],F[v_2]$ and $N[v_2]$ because $v_2$ is a child of $v_1$, but these values are different
 than those stored in $v_s$ because they are associated with a different path from $v_s$ to $v_2$.

\begin{definition}\label{def:feasible}
	Let $v_i$ and  $v_j$ be two nodes  in the tree $\mathcal{T}$ and let $v_j$ be a child node of $v_i$.
	We say that $v_j$ is \emph{feasible} for $v_i$  if $F[v_j] \leq P_f$.
\end{definition}
In the above definition,  $F[v_j]$ is the value stored in $v_i$, i.e., $v_j$ is feasible for $v_i$ if the estimated
probability of failure for a solution path $\mathcal{P}$ going through $v_i$ and then $v_j$ does not exceed the assigned
value $P_f$.

\subsection{Online algorithm alternating planning and execution}

The online algorithm we propose alternates planning and execution (see Algorithm \ref{algo:general}). 
Throughout its planning phase, the search tree expansion is conditioned on the residual budget $B$
which is updated after each action is selected and executed.
 At the first iteration, the algorithm solves the SOPCC  with the assigned budget $B$ and root node $v$ set to $v_s$ (line \ref{algo:start}.)
The solution of SOPCC defines the first action to take, i.e., the robot moves from $v_s$ to the vertex $next_v$
returned by algorithm MCTS-SOPCC (line \ref{sopcccall}) and incurs  a random travel cost $\xi_v$ (line \ref{algo:movecost}.)
 The budget is then updated by setting $B$ to
$B-\xi_v$ (line \ref{algo:budgetupdate})
and the SOPCC is solved again with the updated budget $B$ and  starting  at vertex $v$ (now set to $next_v$ in line \ref{algo:vertexupdate}.) 
The process repeats
until either the final vertex $v_g$ is reached, or the budget is completely spent. In this last case, the
run is considered a failure (line \ref{algo:failure}).  

\begin{algorithm}
\KwData{Graph $G=(V,E)$, vertices $v_s,v_g\in V$, budget $B$}
$v \leftarrow v_s$\; \label{algo:start}
\While{$B>0$ \KwAnd $v\neq v_g$ }{
$next_v \leftarrow \textrm{MCTS-SOPCC}(v,B)$\; \label{sopcccall}
move to vertex $next_v$ and let $\xi_v$ be the incurred cost\;  \label{algo:movecost}
$B \leftarrow B-\xi_v$\; \label{algo:budgetupdate}
$v \leftarrow next_v$\; \label{algo:vertexupdate}
}
\eIf {$B>0$}{\KwRet Success}
{\KwRet Failure} \label{algo:failure}
\caption{Alternating Planning and Execution}
\label{algo:general}
\end{algorithm}

 The first advantage of this solution is that differently from our former works
\cite{CarpinRAL2021b,CarpinIROS2021b,CarpinIROS2021a} it is neither necessary to discretize the temporal dimension
to build a CMDP with a finite state space, nor it is necessary to numerically approximate the transition probabilities
between states.

\subsection{MCTS-SOPCC}
The MCTS-SOPCC algorithm to solve the SOPCC (line \ref{sopcccall} in algorithm \ref{algo:general}) customizes the 
general  MCTS approach  described in section \ref{sec:background} as follows.

{\bf Tree Policy}: 
The tree policy is used to traverse the tree from the root to a leaf and then select a new node to possibly add to the
tree. It is implemented by applying recursively the following strategy inspired by UCT. Assuming $v_i$ is the current vertex,
to each of its neighbors\footnote{From the set of neighbors we exclude the set of vertices already visited because
revisiting an already visited vertex does not give any reward and is therefore useless. In this context,
the terms neighbor and descendant are to be considered synonyms.} $v_j$ we associate the following quantity
(UCTF stands for UCT with Failure):
\begin{equation}\label{eq:uctf}
UCTF(v_j) = Q[v_j](1-F[v_j]) + z \sqrt{\frac{\log t}{N[v_j]}}
\end{equation}
where $t$ is the sum of the number of times that the descendants of $v_i$  have been explored already, and $z$ is a constant.
The node with the highest UCTF value is then selected and this process is repeated until a vertex not yet in 
the tree is selected.
As commonly done in  the basic UCT strategy, if  node $v_j$ has not yet been visited its $N[v_j]$ counter is 0
and  we then set its UCTF value to $\infty$, to make sure
all neighbors are visited at least once. The novel term $Q[v_j](1-F[v_j])$ is the \emph{expected utility} 
of moving to $v_j$, obtained by multiplying the   estimated utility $Q[v_j]$ by 
the probability of success $(1-F[v_j])$. 
Note that as failure is a binary random variable, this is the expected utility.
In this way, given two vertices with similar $Q$ values, the
criterion  favors the one with the lowest failure probability. The last term in the UCTF formula 
is borrowed from UCT and encourages the selection of nodes that have been formerly selected fewer times.

{\bf Rollout Policy}:   when a new vertex $v_j$ is added as a child
of  a node $v_i$ already in the tree  (see Figure \ref{fig:rollout}), it is necessary to estimate  both: 1) how much reward one will
collect by expanding the path through $v_j$; 2)   the probability of exceeding the
available budget before reaching the end vertex $v_g$.

\begin{figure}[htb]
	\centering
	\includegraphics[width=0.3\linewidth]{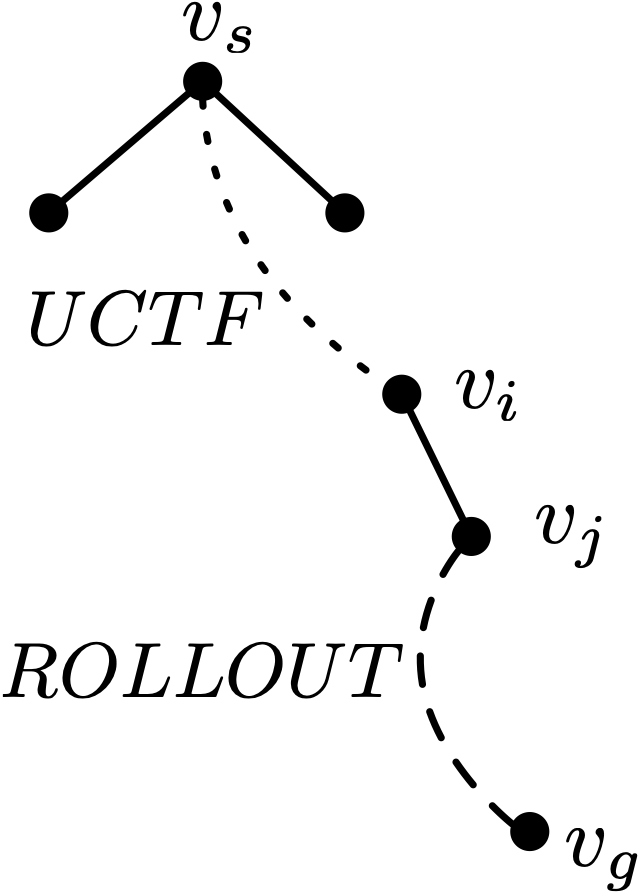}
		\caption{Assuming the tree is rooted in $v_s$, the tree policy UCTF repeatedly selects internal vertices in the tree $\mathcal{T}$ 
		until a vertex leaf $v_i$ is reached, and then a child node $v_j$ is added to the tree.  Then the rollout algorithm is run from $v_j$ to 
		estimate the values $Q[v_j]$ and $F[v_j]$ to be stored in $v_i$ and then propagated back to the root.}
	\label{fig:rollout}
\end{figure}

These values will be stored in $v_i$ as $Q[v_j]$ and $F[v_j]$, respectively, and are computed by the rollout policy
by generating $S$ paths from $v_j$. Each path  is generated using the process sketched in Algorithm  \ref{algo:rollout}.

\begin{algorithm}
	\KwData{leaf vertex $v_j\in V$,  residual budget $B'$}
	$current \leftarrow v_j$\;
	$path \leftarrow \{v_j \}$\;
	\While{{\bf true}}{
	\eIf{random() $< P_R$}{$new \leftarrow randomchild(current)$}{$new \leftarrow greedy(current)$}
	\If{$new ~\neq v_g$}{
		\If{ $\Pr[C(current,new,v_g) < B'] \leq P_f$}{
			append $new$ to $path$\;
			$B' \leftarrow B' - samplecost(current,new)$\;
			$current \leftarrow new$\;
		}
	}
	\Else {
		append $new$ to $path$\;
		return $path$
	}
}
	\caption{Rollout}
	\label{algo:rollout}
\end{algorithm}

To determine the residual budget $B'$ (input parameter in  algorithm \ref{algo:rollout}), we sample the time  $t = \xi_1+\xi_2+\dots+\xi_k$ 
it takes to proceed from the root node $v_s$ to the leaf node	 $v_j$ along the unique path with $k$ edges determined by the
tree policy.
We then set $B'=B-t$ as the residual budget available when starting from $v_j$, where $B$ is the budget available
at the root node $v_s$ (this value is computed before Algorithm \ref{algo:rollout} is executed -- see algorithm \ref{algo:MCTS}).
The rollout procedure builds a path from $v_j$ to $v_g$ by iteratively trying to add vertices either randomly  (line 5),
or greedily (line 7).
Random children are added with probability $P_R$ while the greedy strategy is chosen with probability $1-P_R$. Suitable
$P_R$ values are instance dependent and can be identified with a preliminary grid search. This aspect is further discussed in 
Section \ref{sec:results}.
The greedy strategy works as follows. For all vertices $v_k$ not yet visited, 
compute $r(v_k)/c(current,v_k)$ where $current$ is the last vertex added to the path being built.
Then, discard the vertices such that $\Pr[c(current,v_k)+c(v_k,v_g)> B' ]> P_f$, and pick the one with the highest ratio.
If all vertices are discarded, the greedy step returns $v_g$.
To determine the vertices to discard, we generate $M$ samples for the costs of the edges and use these values
to estimate the probability of exceeding the budget $B'$. 
This is an instance of  the SAA approach described in Section \ref{sec:background}.
The greedy step 
adds to the path the vertex $v_k$ with the highest ratio between reward and cost, but constrained on having
estimated that the probability of moving from $v_k$ to the terminal $v_g$ does not exceed the failure probability $P_f$.
Note that the greedy step may select $v_g$ as the most suitable node to visit next. 
If the newly selected vertex $new$ is the goal vertex $v_g$, we add it to the path and terminate the rollout returning the path (lines 14 and 15).
Otherwise (lines 9 through 12), we consider the path with edges $(current,new)$ and $(new,v_g)$ and we evaluate if the probability that the
cost of this path exceeds $B'$ violates the constraint $P_f$ or not. If it violates the constraint, we reject $new$, otherwise, 
we add it to the path and update $B'$. To determine the probability of violating the constraint, we generate $M$ samples, 
as in the greedy step. It is easy to see that the rollout step is guaranteed to terminate returning a path.

{\bf Backup Policy}: After the $S$ paths from $v_j$ have been generated by the rollout procedure,
we can compute the expected return $Q[v_j]$ and estimate $F[v_j]$, i.e., the probability of failure of expanding the route from $v_i$
through $v_j$. The values are computed by  averaging the rewards and failures\footnote{
To estimate $F[v_j]$ we  take the ratio between the number of paths with cost exceeding the budget 
and the total number of paths $S$.} associated with the $S$ paths returned, as per the SAA approach.
These values must then be propagated upwards through the tree towards the root, as the paths generated from $v_j$
are subpaths of paths starting from the root and therefore influence the values for the  $Q$ and $F$ labels
associated with the children of the root note. In MCTS without constraints, this is often done by averaging the returns, but 
the case considered in this work is different because of the need to eventually select actions leading to paths whose
probability of violating the budget constraint does not exceed $P_f$. 
The backup step is then applied from node $v_j$ backwards towards the root  considering 
the relationship between the $Q$ and $F$ labels of the newly added node $v_j$,	 and the $Q$ and $F$ labels of its
parent node $v_i$. The step also involves the parent of $v_i$, if it exists (called $v_k$ in the following). We refer to figure
\ref{fig:backup} while explaining the process, and algorithm \ref{algo:backup} sketches the procedure (for the sake of brevity
the algorithm just shows the updates $Q$ and $F$, as updates for $N$ are just increases by one.) After the values are propagated
from $v_j$ to  $v_i$, the process is  recursively repeated backward starting from $v_i$, then from $v_k$, and so on until 
the root node is reached.

\begin{figure}[htb]
	\centering
	\includegraphics[width=0.4\linewidth]{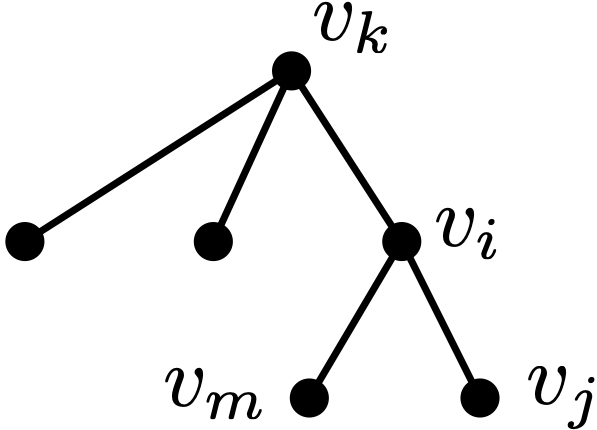}
	\caption{When backing up the values $F[v_j]$ and $Q[v_j]$ stored at node $v_i$, it is necessary to 
		consider their relationships with the values $F[v_i]$ and $Q[v_i]$ stored with $v_k$.}
	\label{fig:backup}
\end{figure}

The backup step at node $v_i$ is as follows (refer to algorithm \ref{algo:backup}).
If $v_i$ is the root, then $v_k$ does not exist, and the backup step does not perform any operation 
because  the values $F[v_j]$ and $Q[v_j]$  are already stored with $v_i$ (the loop in line 3 is never executed.)
If $v_k$ exists, it has its labels $F[v_i], Q[v_i]$  initialized already, and it may be necessary to propagate the
$F$ and $Q$ values upwards if a better path from root to leaf has been found.
We distinguish the following cases:

\begin{itemize}
	\item if $F[v_i] < P_f$ (this value is stored in $v_k$ -- test at line \ref{vifeasible}) it means $v_i$ is feasible for $v_k$. 
	The  new $F$ and $Q$ values determined for $v_j$ should be propagated
	up only if the new path  extending from $v_i$ through $v_j$ is both 1) feasible; 2) provides a better return. The first condition 
	is checked in line \ref{vjfeasible}, while the second condition is checked in line \ref{vjbetter}. To this end, we compare the value $Q[v_i]$ stored in $v_k$
	with the expected value of the new path found, which is $r(v_i)$ (reward obtained visiting $v_j$) plus $Q[v_j]$ (reward obtained
	from $v_i$ onwards). 
	\item  if $F[v_i] > P_f$ it means that the best path  from $v_k$ through $v_i$ violates the chance constraint. In this case, if 
	$F[v_i] > F[v_j]$ then a better path from $v_k$ through $v_i$ has been determined, irrespective of the $Q$ values, because it
	lowers the failure probability (possibly bringing it below $P_f$.) In this case, the values of the new path from $v_k$ to $v_i$ and
	then to $v_j$ are propagated to $v_k$ (lines 10 and 11.)
\end{itemize}
After the values are propagated from $v_i$ to its parent $v_k$, the same mechanism is recursively applied one level up (lines
12 and 13).

\begin{algorithm}
	\KwData{$v_j,Q[v_j],F[v_j]$}
	$v_i \leftarrow v_j.parent$\;
	$v_k \leftarrow v_i.parent$\;
	\While{$v_k$ {\bf is not null}}{
			\If{$v_k.F[v_i] < P_f$ }{ \label{vifeasible}
				    \If{{$v_i.F[v_j] < P_f$}}	{   \label{vjfeasible}
				    	\If{$v_k.Q[v_i] < v_i.Q[v_j]  + r(v_i)$}	{   \label{vjbetter}
				    	 $v_k.Q[v_i] \leftarrow v_i.Q[v_j]  + r(v_i)$\;
				    	$v_k.F[v_i] \leftarrow v_i.F[v_j]$\;
				    }
				}
			}
			\ElseIf{$v_k.F[v_i] > v_i.F[v_j]$ } {
				$v_k.F[v_i] \leftarrow v_i.F[v_j]$ \;
			    $v_k.Q[v_i] \leftarrow v_i[v_j]  + r(v_i)$\;
			}				 
			$v_k \leftarrow v_k.parent$\;
		$v_i \leftarrow v_i.parent$
		}
	\caption{Backup}\label{algo:Backup}
	\label{algo:backup}
\end{algorithm}

{\bf Action Selection:} After the tree $\mathcal{T}$ has been built, the best action available from the root node $v$ is selected.
The  best action is defined  as the feasible node $v_j$ with the highest value $Q[v_j]$ (recall definition \ref{def:feasible}).
If among the children of the root node no feasible node is found (i.e., all nodes connected to the root have an $F$ value exceeding 
the failure probability $P_f$), then action selection returns $v_g$, i.e., it tries to move the robot
to the final vertex $v_g$ in the orienteering graph $V$.   

Algorithm \ref{algo:MCTS} shows how the components described above are put together. 
As  in MCTS algorithms, the tree is expanded through a fixed number of iterations $K$.
At each iteration, the UCTF criterion is used as a tree policy to move from the root
of the tree to a node $v_j$  (line \ref{alg:treepolicy}) that  is  added to the tree if it is not already present.
Then, $S$ paths are generated from $v_j$ (line \ref{alg:s_steps}) using the rollout 
process formerly described (line \ref{alg:rollout}). At each rollout step, the
algorithm considers a different residual budget $B'$ obtained by sampling the time $t$
to move from the root of the tree to the new node $v_j$ (line \ref{alg:sampletime}).
After the $S$ samples are collected, the values $Q[v_j]$ and $F[v_j]$ for $v_j$ can be computed
(line \ref{alg:qf}) and propagated back to the root with the backup procedure (line \ref{alg:update}).
As the last step, the actions counter $N$ is also propagated back to the root (line \ref{alg:counterbackup}.)
Finally, the algorithm returns the action from the root with the highest $Q$ value among the
feasible ones (line \ref{alg:selection}). If no feasible action is found, $ActionSelection$ returns
the goal vertex $v_g$.  Alternatively, when no action is found, one could further extend the tree by running SOPCC again and extending
the current tree rather than restarting from scratch, though this is not done in the experiments we consider
in this paper.

\begin{algorithm}
	\KwData{start vertex $v$,  budget $B$}
	Initialize tree $\mathcal{T}$ with root equal to $v$\;
	\For { $K$ iterations}{ \label{alg:mainiteration}
		$v_j \leftarrow  UCTF(v)$\; \label{alg:treepolicy}
			\For { $S$ iterations}{ \label{alg:s_steps} 
				$t \leftarrow $ SampleTraverseTime($v,v_j$)\; \label{alg:sampletime}
				$B' \leftarrow B - t$\;
				$path \leftarrow $rollout($v_j,B'$)\; \label{alg:rollout}
			}
		compute $Q[v_j]$ and $F[v_j]$ based on the $S$ rollouts \; \label{alg:qf}
		\eIf{$v_j$ is not a child node of $v$}{add $v_j$ as child node of $v$ with  $Q[v_j]$ and $F[v_j]$\; \label{alg:addnode}}
		{
			update estimates of  $Q[v_j]$ and $F[v_j]$ stored in $v$\;
		}
	    Backup($v_j,Q[v_j],F[v_j]$)\; \label{alg:update}
	    BackupN($v_j$)\; \label{alg:counterbackup}
}
\KwRet ActionSelection($root(\mathcal{T})$)\; \label{alg:selection}
	\caption{MCTS-SOPCC}\label{algo:MCTS}
\end{algorithm}

\section{Theoretical Properties}
\label{sec:thprop}
In this section we provide some theoretical properties of the algorithm we proposed. 
We start asserting the asymptotic probabilistic convergence of the method we defined, and then 
analyze the error probability as a function of the finite number of iterations $K$.

 \subsection{Asymptotic Convergence} 

The following theorem characterizes the asymptotic behavior of the MCTS-SOPCC algorithm, i.e., its performance for
large values of the number of independent rollouts $S$. The results build upon the fact that
all iterations of the loop at line \ref{alg:s_steps}  in algorithm \ref{algo:MCTS} are independent,
and therefore the $Q$ and $F$ values computed in line \ref{alg:qf} are independent random variables.

\begin{theorem}\label{th:asymp}
	For $\lim_{S\rightarrow \infty}$, if the MCTS-SOPCC algorithm returns a feasible node $v_j$, then with probability $1-F[v_j]$ there is a solution to the
	SOPCC with expected value $Q[v_j]$.
\end{theorem}
\begin{proof}
	
	First, observe that by the weak law of large numbers as the number of independent rollouts $S$ increases, the computed values
	for $F$ and $Q$ converge in probability to their expected values. More specifically, since $F$ is  the indicator variable
	for the event ``exceeding the budget $B$," 	its expectation is then equal to the probability of the  associated event.
	
	By construction, the node $v_j$ returned by MCTS-SOPCC is always an immediate descendant of the
	node $v$ stored at the root of the tree $\mathcal{T}$.
	To prove the claim we separately consider the cases where $v_j$ is a leaf node or not. The theorem 
	applies only to the case where the algorithm returns a feasible node $v_j$. If the root has no feasible child,
	then the algorithm returns $v_g$, as discussed when describing how action selection is implemented. 
	
	If the feasible node $v_j$ is a leaf node in the tree,   the values $Q[v_j]$
	and $F[v_j]$ were stored when $v_j$ was added to the tree and were never modified because $v_j$ has no 
	descendants. These values, by construction,
	were obtained with $S$ independent rollouts from $v_j$ (loop at line \ref{alg:sampletime}, and line \ref{alg:qf} in 
	algorithm \ref{algo:MCTS}), and therefore the claim directly  holds by the law of large numbers, as stated above.
	
	If $v_j$ has one or more descendants, we need to distinguish two cases. If the   $Q$ and $F$ values were never
	modified after $v_j$ was added, then we are in the same case as above and the statement therefore holds.
	If the  $Q$ and $F$ values were modified (one or more times), then that happened during a call to the \emph{Backup} function.
	In \emph{Backup}, the $Q$ and $F$ are modified only to either 1) increase the $Q$ value while ensuring
	that $v_j$ remains feasible (lines 7 and 8 in algorithm \ref{algo:backup}), or 2)  decrease the $F$ value while $v_j$ is unfeasible 
	(lines 10 and 11 in algorithm \ref{algo:backup}). However, since the node is feasible, the last update must have been 
	an update of the first type and the stored $F$ value converges to its mean as $S$ increases.
	
\end{proof}

True to the MCTS spirit, the presented algorithm is an anytime algorithm, i.e., by increasing the value 
of the parameters $K$ and $S$ the quality of the returned solution increases. This is in contrast to our former 
works where a solution is produced only after the whole state space for the CMPD has been 
constructed and the associated linear program solved. In section \ref{sec:results} we will assess how 
$K$ influences the quality of the solution and the computation time. The other relevant parameter
is $S$, the number of samples used to estimate the failure probability of a path from the root to a leaf.
Obviously, the larger the number of samples, the more accurate the estimate and the associated computation time.
As common for these types of algorithms, the probability of not finding the solution (if it exists) tends to 0 as the 
values of $K$ and $S$ increase. The convergence velocity is influenced by the size of the search space, i.e., by the 
average branching factor in the tree.

 \subsection{Error Rates} 

While theorem \ref{th:asymp} provides a guarantee regarding the asymptotic behavior of the algorithm, it is interesting
to reason about the performance and error rates for finite values of $K$, i.e., the number of iterations.
In particular, we are interested in the probability that the final step in algorithm \ref{algo:MCTS} selects the wrong action, 
i.e., it picks a suboptimal vertex. To this end, it is useful to draw a comparison with bandit algorithms and the concept of
regret. In bandit algorithms, one is interested in balancing exploration and exploitation while bounding the cumulative regret.
However, in the problem we consider there is no cost for exploration (done thorough simulation), and the only cost 
to consider is the so called \emph{simple regret}, i.e., the the cost incurred by making the wrong choice based on 
the data available. Hence, the problem we consider is related to a special class of bandit problems known as 
\emph{pure exploration} problems (see \cite{BanditsBook}, chapter 33, as well as our recent paper \cite{CarpinIROS2023b} for more details.)
Without loss of generality, assume that after $\mathcal{T}$ has been built the optimal choice is $v_1$, 
and the values $F[v_1]$ and $Q[v_1]$ are stored in the root note.
The function ActionSelection($root(T)$) will return the wrong action in two mutually exclusive cases:
\begin{itemize}
	\item when $F[v_1] > P_f$, i.e., when the estimate of the failure probability of action $v_1$ exceeds the assigned failure
	probability. Note that since $v_1$ is the optimal choice, it must be feasible, and therefore its failure probability must
	be smaller than $P_f$. Hence this condition corresponds to an overestimation of such failure probability.
	\item when $F[v_1] \leq P_f$ but there is another suboptimal feasible vertex, say $v_2$, such that $Q[v_2] > Q[v_1].$
\end{itemize}
We now separately study both these conditions. We start observing that \emph{failure}, i.e., overrunning the budget, 
is a binary condition, and is therefore modeled by a Bernoulli random variable with parameter $p$. 
The return  associated with an action is instead a non-negative random variable estimated by $Q[v_i]$.

The following theorem characterizes  the probability of wrongfully estimating that the optimal action $v_1$ is infeasible, i.e., 
it provides a bound on $\Pr[F[v_1] > P_f]$ as a function of 
 the number of times $N[v_1]$ that action $v_1$ has been tried  (the first condition defined above). 
Let us define $f(v_1)$ as the parameter $p$ describing the Bernoulli variable modeling the binary event \emph{failure after
	selecting $v_1$.}
Therefore $F[v_i]$ is a random variable estimating $f(v_i)$.
Let $\Delta = P_f - f(v_1)$ and since $v_i$  is feasible, it follows that $\Delta > 0$. Intuitively,   $\Pr[F[v_1] > P_f]$ should shrink
as $N[v_1]$ grows because more samples lead to a more accurate estimate, and 
it should grow when $\Delta$ decreases because
when $f(v_1)$ approaches $P_f$ even small estimation errors may lead to wrongly decide $v_1$ as unfeasible.  
The following theorem formalizes these intuitions.

\begin{theorem}
	Let $F[v_1]$ be the estimate of $f(v_1)$ after the tree $\mathcal{T}$ has been built with $K$ iterations with $N[v_1]$
	attempts for action $v_1$; let $\Delta=P_f-f(v_1)$. Then
	\[
	\Pr[F[v_1] > P_f] \leq \sqrt{\frac{f(v_1)(1-f(v_1))}{2\pi N[v_1] \Delta^2}}e^{\left( -\frac{N[v_1] \Delta^2}{2 f(v_1)(1-f(v_1))}\right)}
	\]
\end{theorem} 
\begin{proof}
	We start recalling the  Cram\'er-Chernoff method (see e.g., \cite{BanditsBook}, chapter 5). Let $X$ be a 
	random variable with $\mu=\mathbb{E}[X]$ and finite variance $\sigma^2$. If $\hat{\mu}=\frac{1}{n}\sum_{i=1}^n X_i $
	where the $X_i$s are a sequence of independent identically distributed samples of $X$, then for $\varepsilon >0$
	
	\begin{equation}\label{eq:chernoff}
		\Pr[\hat{\mu} \geq \mu + \varepsilon] \leq 
		\sqrt{\frac{\sigma^2}{2\pi n \varepsilon^2}}e^{\left( -\frac{n \varepsilon^2}{2 \sigma^2}\right)}
	\end{equation}
	
	Next, observe that  $F[v_1]$ is the estimate of $f(v_1)$ obtained through $N[v_1]$ independent simulations. 
	Moreover, the Bernoulli variable modeling the \emph{failure} event has variance $f(v_1)(1-f(v_1))$.
	Note furthermore that because of how we defined $\Delta$  the event $F[v_1] > P_f$ is equivalent to the event $F[v_1] > f(v_1) + \Delta$.
	Substituting these values in Eq.~\eqref{eq:chernoff}, the claim follows.
\end{proof}

We finally turn to the second source of error, i.e., the case where the algorithm correctly estimates that the optimal action $v_1$ is
feasible but selects a different feasible action $v_2$ because it estimates that it has a higher return. We start observing that while
the failure probability is modeled by a Bernoulli random variable, returns are random variables with positive support
that in general do not have a known distribution. 
Consequently, 
in the analysis the set of tools we can use is much more limited. Chebychev's inequality, stated below, is applicable for any random
variable $X$,  and  for any $\varepsilon > 0$ states
\[
\Pr[|X-\mathbb{E}[x]| \geq \varepsilon] \leq \frac{\mathbb{V}[X]}{\varepsilon^2}
\]
where $\mathbb{E}[x]$ is the expectation of $X$ and $\mathbb{V}[X]$ is its variance. Mirroring the notation introduced
above, let $Q[v_1]$ and $Q[v_2]$ be the random variables estimating the expected return when selecting $v_1$ and $v_2$, respectively.
Moreover, $q(v_1)$ and $q(v_2)$ are the expected returns when selecting $v_1$ and $v_2$.

\begin{theorem}\label{th:errorQ}
	Let $v_1$ be the optimal feasible action and let $v_2$ be a suboptimal feasible action. Then, the probability $P_{ERR}$ that
	$v_2$ is selected instead of $v_1$ is bounded by 
	\[
	P_{ERR} \leq	\frac{\sigma^2_z}{\min\{N[v_1],N[v_2]\}G}
	\]
	where: $\sigma^2_z$ is the variance of the random variable $z = q(v_2)-q(v_1)$ and $G = q(v_1) - q(v_2).$
\end{theorem}
\begin{proof}
	First observe that $Q[v_1]$ is an estimate of $q_1$ obtained by averaging $N[v_1]$ independent samples, and that $q(v_1)$ is its expectation.
	The same relationships hold between $Q[v_2]$ and $q(v_2)$. Next, since we assumed $v_1$ is optimal and $v_2$ is suboptimal, it follows
	that $G = q(v_1) - q(v_2) > 0$. Let us now introduce the new random variable $Z=Q[v_2]-Q[v_1]$. This can be thought as taking the 
	average difference between $\min\{N[v_1],N[v_2]\}$ samples of $Q[v_2]$ and $Q[v_1]$ (as in general the number of samples will be different, 
	we consider the minimum between the two, and discard those in excess.) Therefore, its variance is $\sigma^2_z/\min\{N[v_1],N[v_2]\}$.
	From the linearity of expectation, it follows that 
	$\mathbb{E}[Z]  = q(v_2)-q(v_1)$. Since ActionSelection returns the feasible action with the highest estimated $Q$ value, 
	it will return $v_2$ instead of $v_1$ only if $Q[v_2] > Q[v_1]$, i.e., only if $Z > 0$.  We define $P_{ERR} = \Pr[Z>0]$
	as this is the probability that ActionSelection selects the wrong action.
	Next, consider:
	\[
	Z > 0 \Leftrightarrow Z > q(v_2) - q(v_1) + q(v_1) - q(v_2) \Leftrightarrow Z > \mathbb{E}[Z]+G
	\]
	Hence the claim follows by applying Chebychev's inequality to $\Pr[|Z-\mathbb{E}[Z]|\geq G]$ and recalling that
	$\mathbb{V}[X]= \sigma^2_z/\min\{N[v_1],N[v_2]\}.$
\end{proof}

We conclude by observing how the bound provided in theorem \ref{th:errorQ} depends on the gap $G$. 
For small values of $G$ the bound becomes loose. However,  small $G$ values mean that the simple regret 
incurred by selecting the suboptimal vertex $v_2$ rather than the optimal vertex $v_1$ is small, and therefore 
the impact on performance is limited.
On the contrary,
as $G$ increases (and therefore the simple regret increases), the probability of making the wrong choice decreases.
Additionally, the error probability decreases linearity with the smallest number of samples between $v_1$ and $v_2$.

\section{Results}
\label{sec:results}

\begin{figure*}[htb]
	\centering
	\includegraphics[width = 0.22 \textwidth]{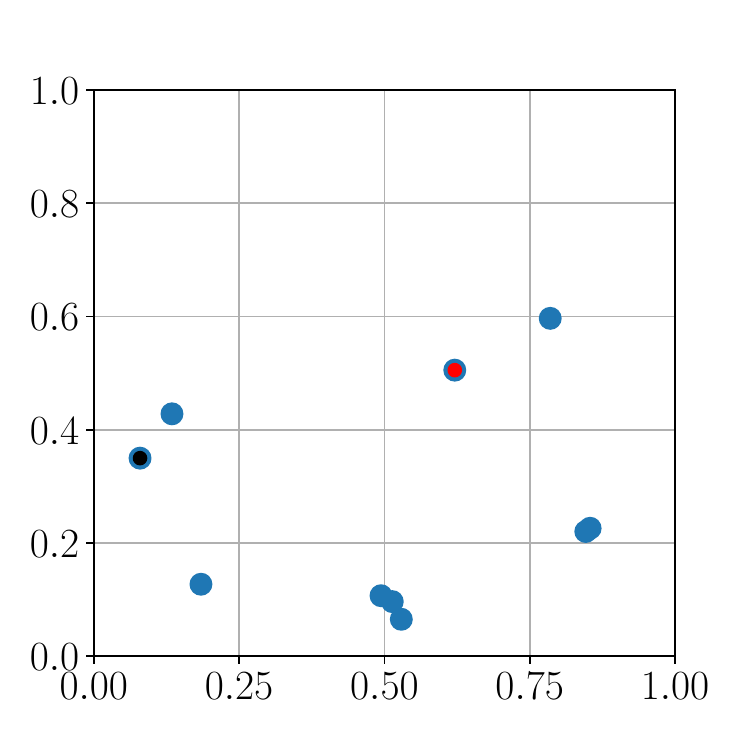}
	\includegraphics[width = 0.22 \textwidth]{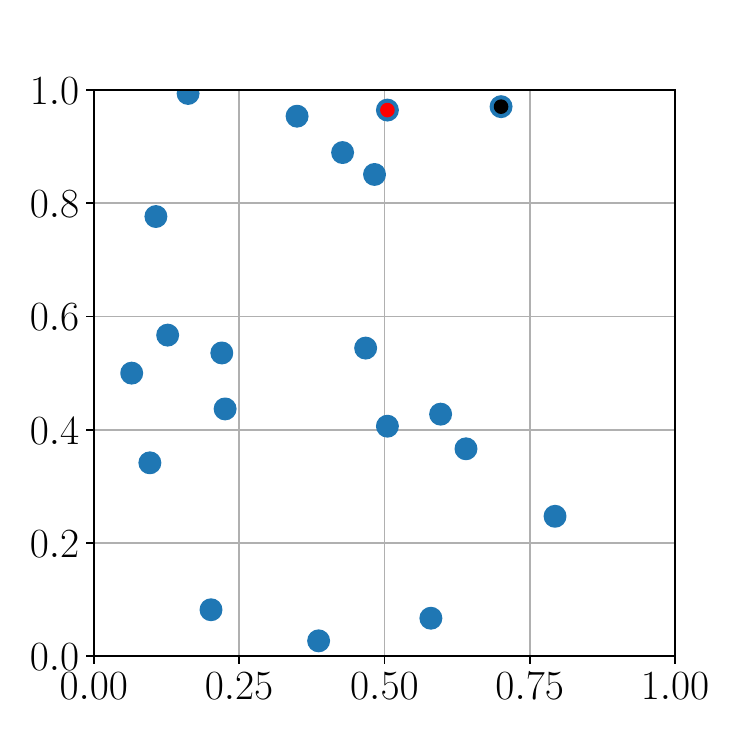}
	\includegraphics[width = 0.22 \textwidth]{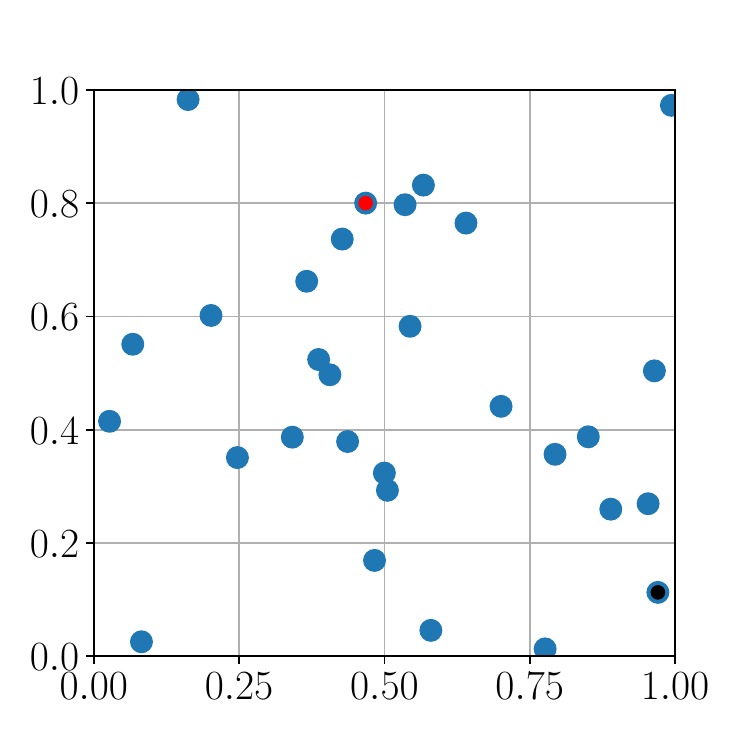}
	\includegraphics[width = 0.22 \textwidth]{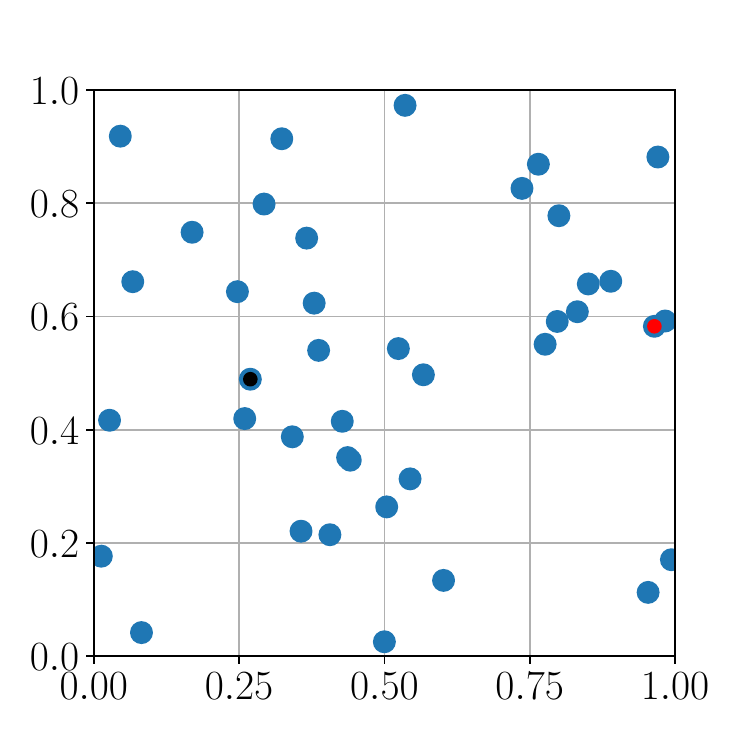}
	\caption{The four test graphs used for our initial benchmarking. In all graphs
		the start vertex is marked in red and the end vertex is marked in black.}
	\label{fig:graphs}
\end{figure*}

In this section, we provide two types of results. First, we assess the sensitivity of the proposed algorithm to the parameters
{ $P_R$ (balance between random exploration and greedy rollout),}
$K$ (number of iterations) and $S$ (number of simulations to estimate the failure probability). 
Next, we make some comparisons with an alternative method formerly proposed in literature { solving both 
	our formerly developed problem instances, but also  third party benchmark problems.\footnote{ All code needed to 
	produce these results and charts is available on {\tt https://github.com/ucmrobotics/StochasticOrienteering}.}} 
In all our tests we kept the parameter $z$ (coefficient in the UCTF quantity defined in Eq.~\eqref{eq:uctf}) equal to 3.
This value was selected after a few preliminary tests showing this is a good compromise between exploration and exploitation.
{
All numerical tests were executed on a system with an Apple M1 Max processor and 32 GB of RAM, and  our MCTS code is written in Python.
}

To ease the comparison with our  previous work { \cite{CarpinRAL2021b,CarpinIROS2021b,CarpinIROS2021a} we start considering} 
the same setup to generate test cases. The $n$ vertices in the graph are randomly sampled inside the unit square,
and rewards are sampled from a uniform distribution with support $[0,1]$. All graphs are complete, i.e., $(v_i,v_j) \in E$
for each $v_i \neq v_j$. Figure \ref{fig:graphs} shows these benchmark cases.

The random cost associated with edge $(v_i,v_j)$ is
\begin{equation}\label{eq:noise}
\kappa d_{i,j} + \mathcal{E}\left( \frac{1}{(1-\kappa)d_{i,j}}\right)
\end{equation}
where $d_{i,j}$ is the Euclidean distance between $v_i$ and $v_j$ and $\mathcal{E}(\lambda)$  is a random 
sample from the exponential distribution with parameter $\lambda$. This formulation ensures
that the expected cost to traverse $(v_i,v_j)$ is equal to $d_{i,j}$ and the cost is non-negative. In all our 
experiments the parameter $\kappa$ is set to 0.5. Before discussing the results, it is worth
outlining that complete graphs are the most challenging to deal with because every node in the tree
has the maximum possible branching factor, and for $n$ vertices there are $\mathcal{O}(n!)$ possible paths\footnote{The number is less than $n!$ because  the start and final vertices are assigned.} in the  space of possible policies.

{
	Our first test aims at identifying a suitable value for the $P_R$ parameter used in Algorithm \ref{algo:rollout}, i.e., the
	probability of picking a random vertex for path expansion instead of doing a greedy choice. To this end, we
	performed a grid search for different values of $P_R$ ranging from 0.1 to 0.9 over the four test environments
	displayed in figure \ref{fig:graphs}. Results are shown in Figure \ref{fig:pr_trend} where we display the average
	normalized reward, i.e., the ratio between the maximum reward obtained in each graph instance and the reward 
	obtained for a given choice of $P_R$. This normalization is necessary because as the number of vertices varies, 
	the rewards  vary, too, and therefore comparing the absolute values can be misleading. 
	This preliminary assessment shows a mild sensitivity on $P_R$ (note  the scale of the $y$ axis) and consequently 
	in the following we fixed $P_R = 0.3$ in all our subsequent experiments.

\begin{figure}[h!]
	\centering
	\includegraphics[width=\linewidth]{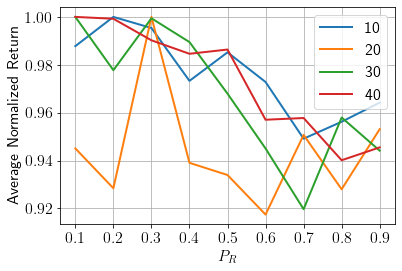}
	\caption{ Average normalized return for different values of $P_R$ on test graphs with the number of vertices
		ranging from 10 to 40. Averages are computed over 100 independent runs for each   $P_R$ value.}
	\label{fig:pr_trend}
\end{figure}
}

Figure \ref{fig:time_trend} shows how the computational time grows with the number of iterations $K$
(red line) as well as the standard deviation. The figure was collected for a problem instance with
20 vertices and  averaged over 50 independent executions. The trend is roughly linear,
as expected, and therefore one can accordingly tune the number of iterations based on the available time.
Variations in the computation time emerge because the produced paths may have more or less vertices
depending on the realizations of the stochastic travel times. It is worth observing that this time is not
all spent upfront but rather distributed along the path (see algorithm \ref{algo:general}).

\begin{figure}[htb]
	\centering
	\includegraphics[width=\linewidth]{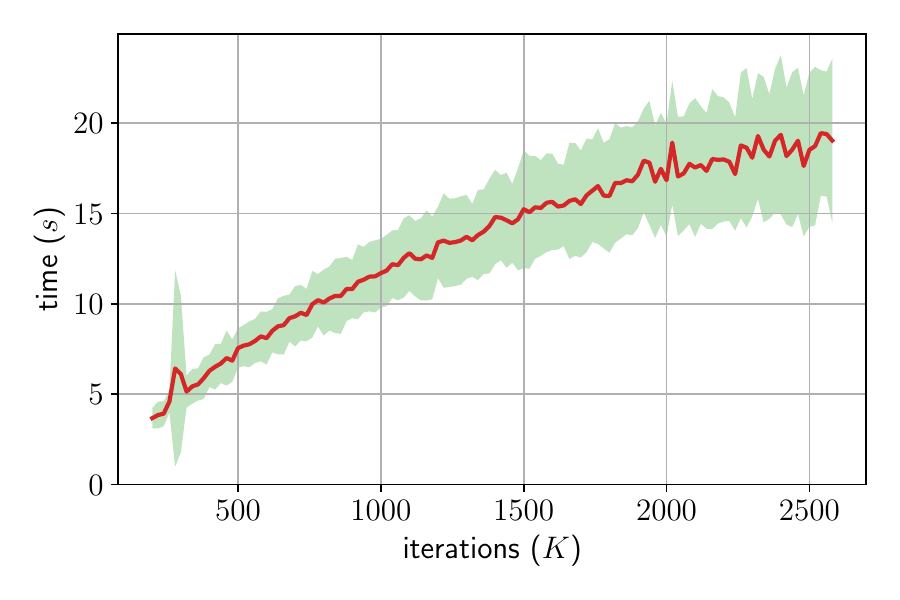}
	\caption{Computation time as a function of the number of iterations $K$ (red line). The green area
	shows the standard deviation (data averaged over 50 trials for each value of $K$). The chart
refers to a graph with 20 vertices and $S=100$ samples.}
	\label{fig:time_trend}
\end{figure}

Next, we investigate how the number of iterations $K$ in algorithm \ref{algo:MCTS} influences
the collected reward.  Figure \ref{fig:performance_trend} shows how the amount of collected 
rewards changes with $K$ (with all other parameters fixed.) for a graph with 20 vertices (red line)
 30 vertices (blue line) and 40 vertices (orange line). Data is averaged over 50 runs
 with $K$ varying from 200 to 2580.
In all instances the reward barely grows with the number of iterations, showing that already with
a value of $K$ around 1000 the algorithm displays a good performance.
For larger graphs with a larger branching factor, with larger values of $K$ one can expect a continued increase in the
accrued reward, but this comes at the cost of increased computational time, as shown in figure \ref{fig:time_trend}.
Overall, this figure seems to indicate that the reward collected by the  algorithm is not too sensitive to the value of $K$ (this would
of course not be the case when $K$ is decreased to smaller values now shown in the figure, as it would not have
the ability to sufficiently explore the set of paths.)

\begin{figure}[htb]
	\centering
	\includegraphics[width=\linewidth]{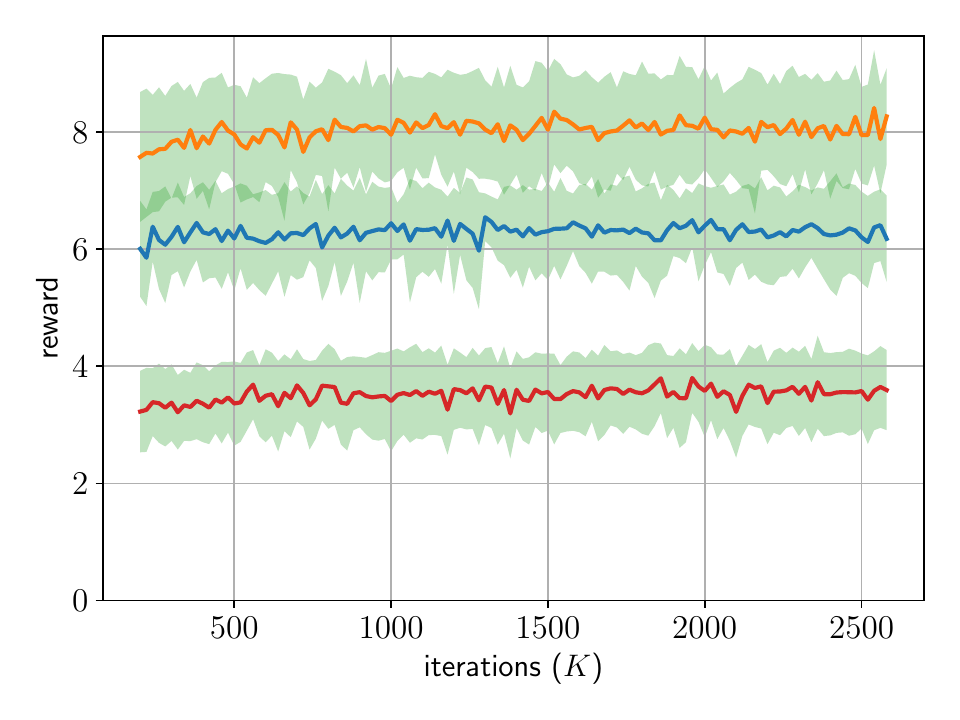}
	\caption{Reward as a function of the number of iterations $K$. The orange line 
	shows the trend for   40 vertices, the blue shows the reward trend for 30 vertices and the 
	red shows the trend for 20 vertices.
	The green area 	shows the standard deviation (data  averaged over 50 trials for each value of $K$).
	In all instances $S=100$ samples.}
	\label{fig:performance_trend}
\end{figure}

Finally, figure \ref{fig:samples_trend} shows how the probability of exceeding the budget
$B$ varies with the number of samples $S$ used to estimate the time to traverse a path.
In this specific case, the assigned failure probability was $P_f = 0.1$. As the number of samples
increases, the probability of failure decreases, as expected. This will further decrease
as the parameter $K$ increases because a larger part of the search space is searched 
(the chart was produced with $K=1000$, as per the considerations discussed above.)

\begin{figure}[htb]
	\centering
	\includegraphics[width=\linewidth]{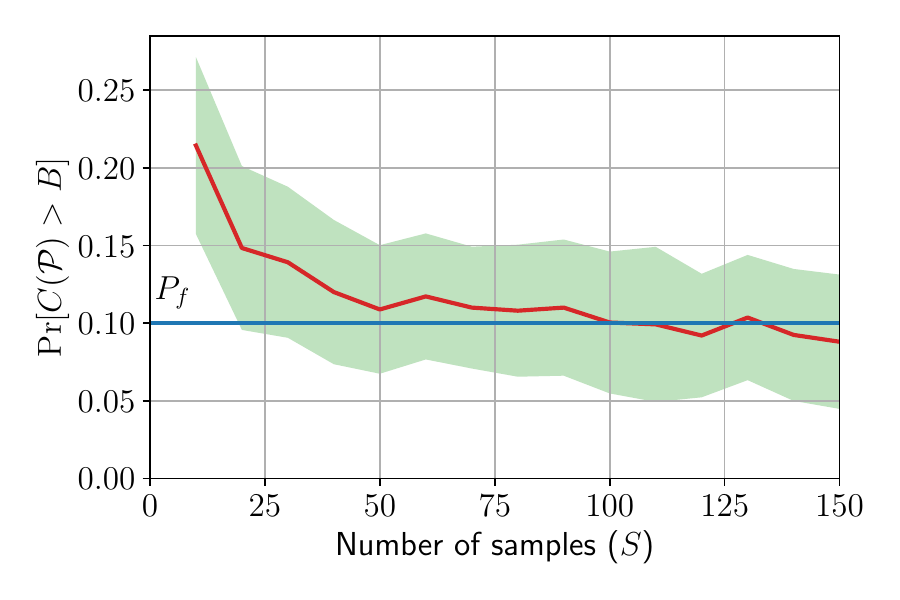}
	\caption{Probability of failure as a function of the samples size $S$ for a graph with 20 vertices, $K=1000$
		and $P_f=0.1$. The green area 	shows the standard deviation (data  averaged over 50 trials for each value of $S$).}
	\label{fig:samples_trend}
\end{figure}

Having assessed the sensitivity of the algorithm to its parameters, we next contrast its performance to alternative
solutions. In \cite{CarpinCASE2022} we already showed that the approach presented in this paper outperforms CMDP formulations we studied in 
 \cite{CarpinRAL2021b,CarpinIROS2021b,CarpinIROS2021a}. Therefore, in this venue we instead compare it with 
  the method proposed in \cite{10.1007/978-3-642-41575-3_30} that uses a MILP formulation to solve the problem.
This solution is  offline, i.e., it computes the path before its execution.  
{
The  formulation is given in the subsequent box. In particular, 
compare   Eq.~\eqref{eq:chanceMILP} with Eq.~\eqref{eq:sopcc} and the problem definition provided in section \ref{sec:background}.}

\begin{MILP}
	
\begin{align}
	\max_{\pi} & \sum_{i,j}\pi_{ij}r(v_i) \nonumber \\
	\textrm{s.t.} \nonumber \\
	\pi_{ij} &\in \{0,1\}  \quad \forall v_i,v_j  \in V  \nonumber \\
	\sum_{j} \pi_{ji} \leq 1 & \quad 	\sum_{j} \pi_{ij} \leq 1  \quad \forall v_i \in V   \nonumber \\
	\sum_{j} \pi_{1j} = 1 & \quad 	\sum_{j} \pi_{jn} = 1 \nonumber  \\
	\sum_{j} \pi_{ij} &-  	\sum_{j} \pi_{ji} = 
	\begin{cases}
		1 & \text{if} ~i = 1  \nonumber\\
		-1 & \text{if} ~i = n  \nonumber\\
		0 & \text{otherwise} \nonumber
	\end{cases} \nonumber \\
	s_i \leq s_j & -1 + (1-\pi_{ij})M \quad \forall v_i,v_j  \in V  \nonumber \\
	s_1 = 1 \quad&  s_n = n \quad s_i \in [1,n] \quad \forall v_i \in V \nonumber  \\
	\Pr & \left( \sum_{i,j} \pi_{ij}\xi_{ij} \geq B\right) \leq P_f \label{eq:chanceMILP}
\end{align}

\end{MILP}

The formulation introduces one binary variable $\pi_{ij}$ for each edge $(v_i,v_j)$ and aims at maximizing the sum
of collected rewards. The reader is referred to  \cite{10.1007/978-3-642-41575-3_30}  for  a full discussion 
of the other variables and constraints. For the sake of comparison, here we focus on Eq.~\eqref{eq:chanceMILP} which involves the 
chance constraint because this influences the comparisons we will present next.  
As pointed out in  \cite{10.1007/978-3-642-41575-3_30}, this constraint is in general
non linear, and therefore the SAA method discussed in Section \ref{sec:background} is used to 
turn this nonlinear constraint into a linear deterministic approximation. 
This approximation turns the optimal formulation given above into a suboptimal one
that can be solved by a MILP solver. 
The approximation 
is as follows.
First, draw $Q$ samples for each random variable $\xi_{ij}$.
These samples are indicated as $t_{ij}^q$. Integer variables $z^q$ for each sample are introduced, 
together with these additional linear constraints:
\[
z^q \geq \frac{\sum_{ij}\pi_{ij}t_{ij}^q-B}{B} \qquad z^q \in \{0,1\}
\]
Starting from these definitions, the chance constraint  in Eq.~\eqref{eq:chanceMILP} is 
substituted with the following 
\[
\frac{\sum_q z^q}{Q} \leq \beta
\]
where, as noted in \cite{10.1007/978-3-642-41575-3_30}, the value $\beta$ is set by the user and is smaller
than the $P_f$ used in  Eq.~\eqref{eq:chanceMILP}. This approach therefore leaves open the questions of
1) how to choose a $Q$ value; 2) how to  pick $\beta$.

Our implementation of the MILP solution uses the Gurobi software. After preliminary experiments
aiming at finding suitable values for $Q$ and $\beta$,  in all cases we set $Q = 80$ and $\beta= P_f/2$.
Moreover, in all tests we cap the computation time at 600 seconds. If the MILP solver has not found a solution
by the deadline, the best solution produced thus far is returned.
For the MCTS implementation, instead, we set $K=350$ (number of tree expansions) and
$S=100$ (number of rollouts per leaf node.) These values were selected based on the results we presented
at the beginning of this section.

The comparisons we present span three dimensions. The first is the collected reward, the second is the failure 
probability, and the third is the computational time. Figure \ref{fig:reward_milp_mcts} compares the reward
collected by the two algorithms. More specifically, it is the ratio between the reward collected by the MCTS
algorithm and the MILP algorithm. MCTS results are averaged over  100 independent runs, while MILP results are averaged over 10 runs (we
run fewer test cases for MILP because, as it will be shown  later, it takes much longer than the MCTS.)

\begin{figure}[htb]
	\centering
	\includegraphics[width=0.9\columnwidth]{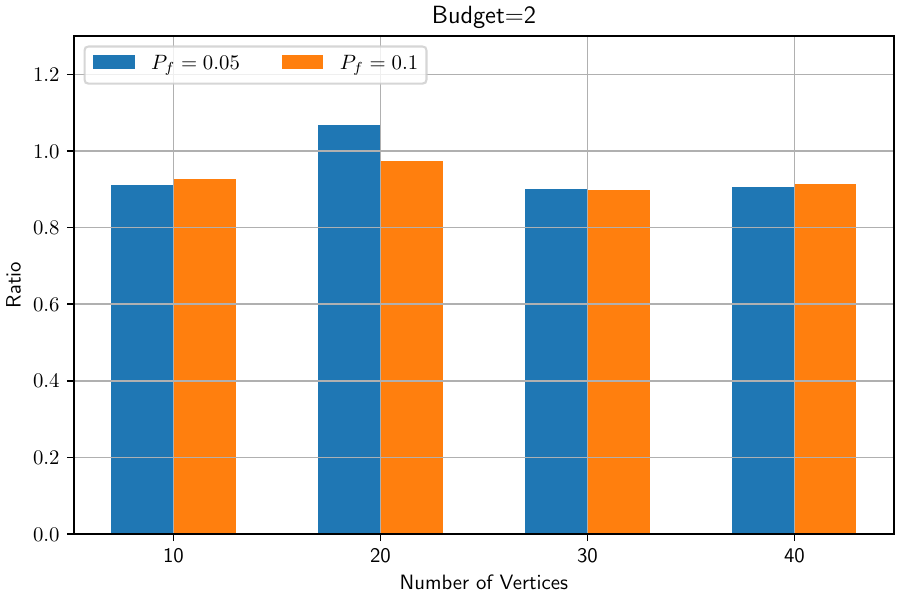}
		\includegraphics[width=0.9\columnwidth]{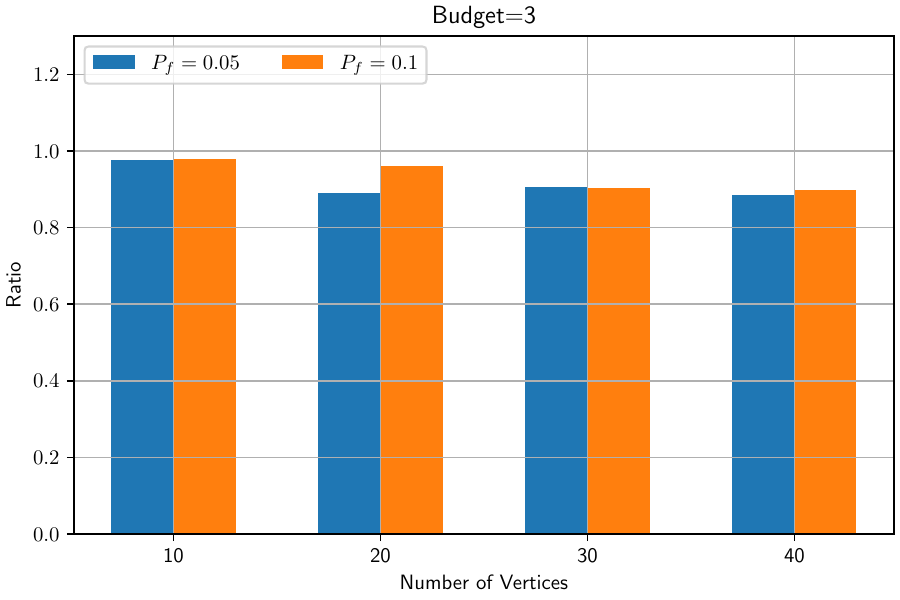}
		\caption{Ratio between the reward collected by the MCTS and the MILP solutions for different values of $P_f$.
		The top panel is for $B=2$ and the bottom one is for $B=3$. }
			\label{fig:reward_milp_mcts}
\end{figure}

As the MILP algorithm approximates the optimal solution, the ratio is expected to always be below 1
and the closer to 1, the better for the MCTS algorithm. The figure compares two different budgets (top panel
$B=2$ and bottom panel $B=3$), as well as two different values for $P_f$. As  can be seen, the ratio is always
in the range of 0.9 and in some cases closer to 1. A special case to be observed is for $B=2$ and $P_f = 0.05$, where
the MCTS algorithm beats the MILP algorithm and the ratio is therefore larger than 1. While this is in general not possible, one has to keep in mind
that in all cases the MILP algorithm was stopped after 600 seconds. Hence in this specific instance, it can be inferred that the MILP
algorithm was stopped while the partial solution was still relatively far from the optimal one, so that MCTS managed to find a better one.

Next, we compare the effective failure probability obtained by the two algorithms. Results are shown in figures 
\ref{fig:failure_milp_mcts2} and \ref{fig:failure_milp_mcts3} for different budgets and failure probabilities.
In this test,  data are averaged over 50 trials for each algorithm. More precisely, for the MCTS algorithm we consider
50 independent runs, while for the MILP algorithm we consider 10 solutions and run each of these 5 times.

\begin{figure}[htb]
	\centering
	\includegraphics[width=0.9\columnwidth]{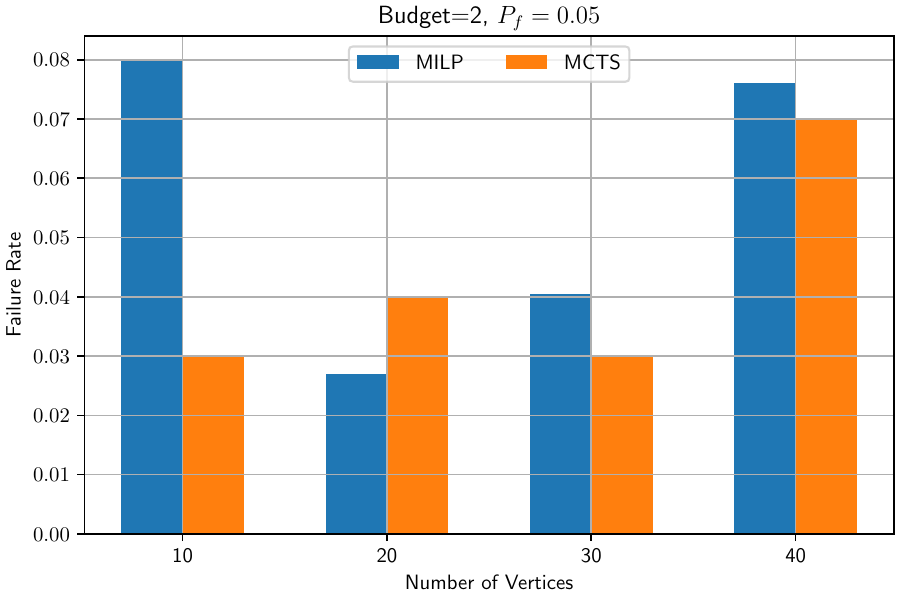}
	\includegraphics[width=0.9\columnwidth]{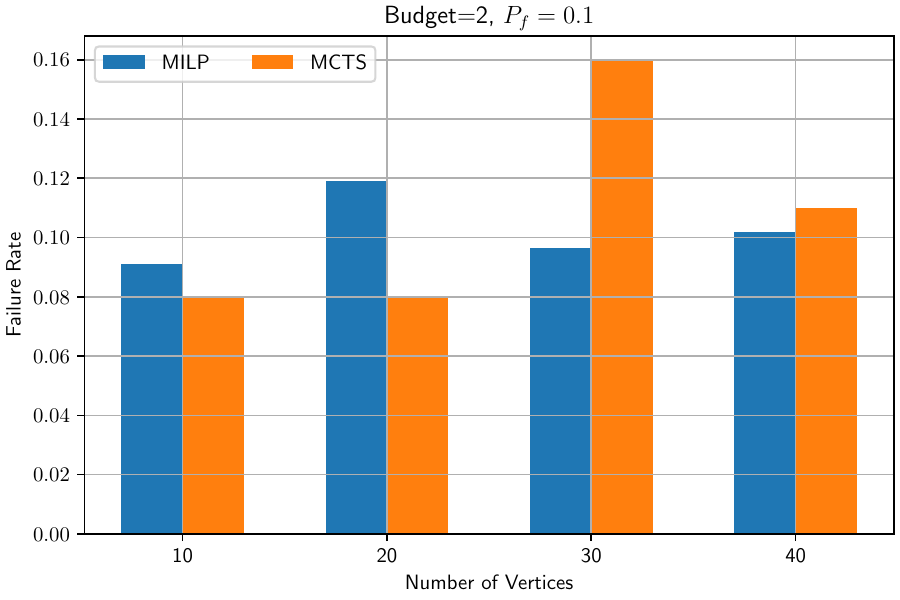}
	\caption{Failure rates incurred by the MILP algorithm (blue) and MCTS algoritm (orange) for 
	the case $B=2$. The top panel is for $P_f=0.05$ and the bottom one is for $P_f=0.1$. }
	\label{fig:failure_milp_mcts2}
	\end{figure}
	
	\begin{figure}[htb]
		\centering
		\includegraphics[width=0.9\columnwidth]{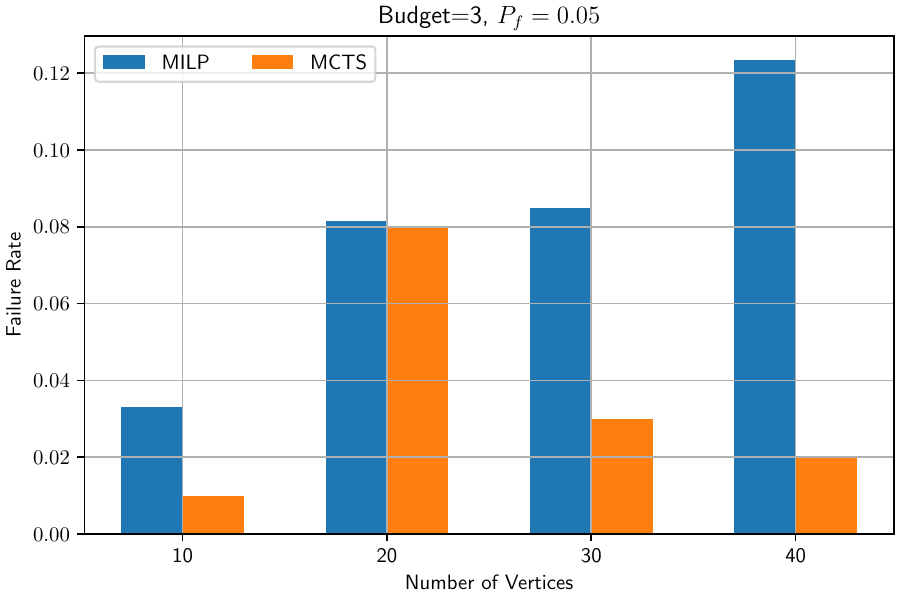}
			\includegraphics[width=0.9\columnwidth]{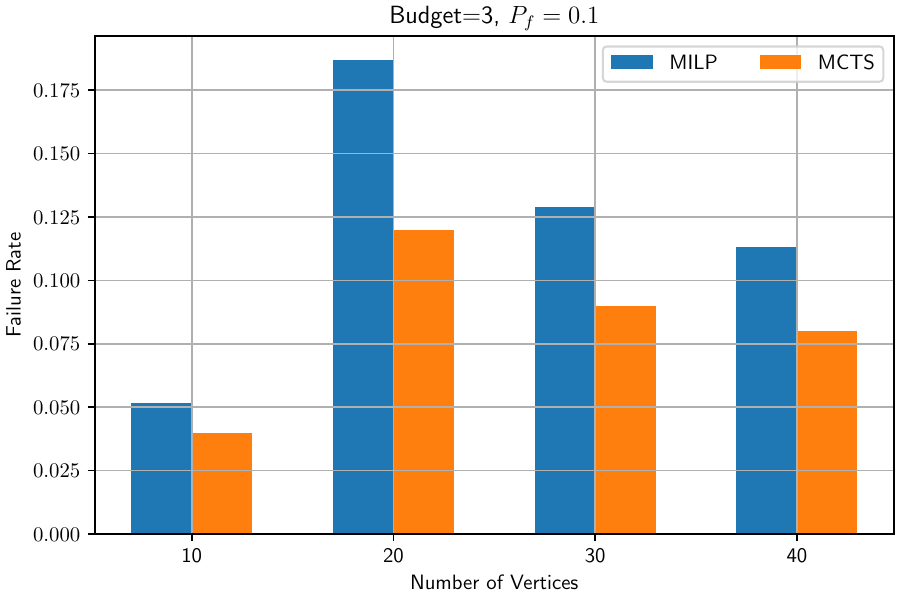}
	\caption{Failure rates incurred by the MILP algorithm (blue) and MCTS algoritm (orange) for 
		the case $B=3$. The top panel is for $P_f=0.05$ and the bottom one is for $P_f=0.1$. }
		\label{fig:failure_milp_mcts3}
\end{figure}

In this case, lower is better, and we observe that 
in 13 out of 16 cases MCTS achieves a lower failure probability than the MILP algorithm.
As pointed out by the authors, the MILP algorithm is sensitive to the values selected for $Q$ and $\beta$
and its failure probability could decrease by increasing $Q$ and decreasing $\beta$. However, both 
changes come at a cost. Increasing $Q$ increases the number of binary variables (and hence computation time)
and decreasing $\beta$ will generate routes collecting less reward. The MCTS algorithm, instead, does not need
a careful tuning of the parameters, as pointed out at the beginning of this section.

Finally, we compare the computational time again for the same values for $B$ and $P_f$. Results are shown
in figures \ref{fig:time_milp_mcts2} and \ref{fig:time_milp_mcts3}.

\begin{figure}[htb]
	\centering
	\includegraphics[width=0.9\columnwidth]{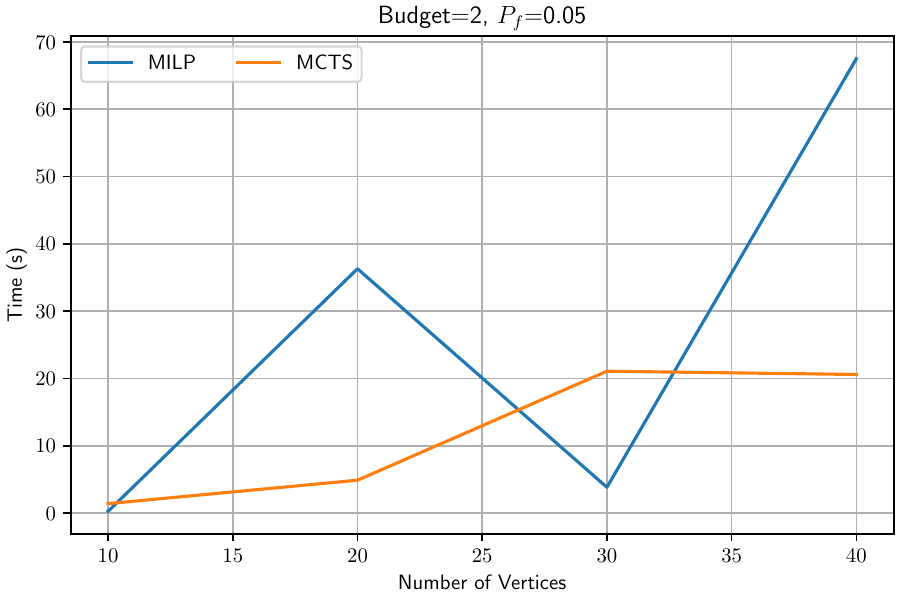}
	\includegraphics[width=0.9\columnwidth]{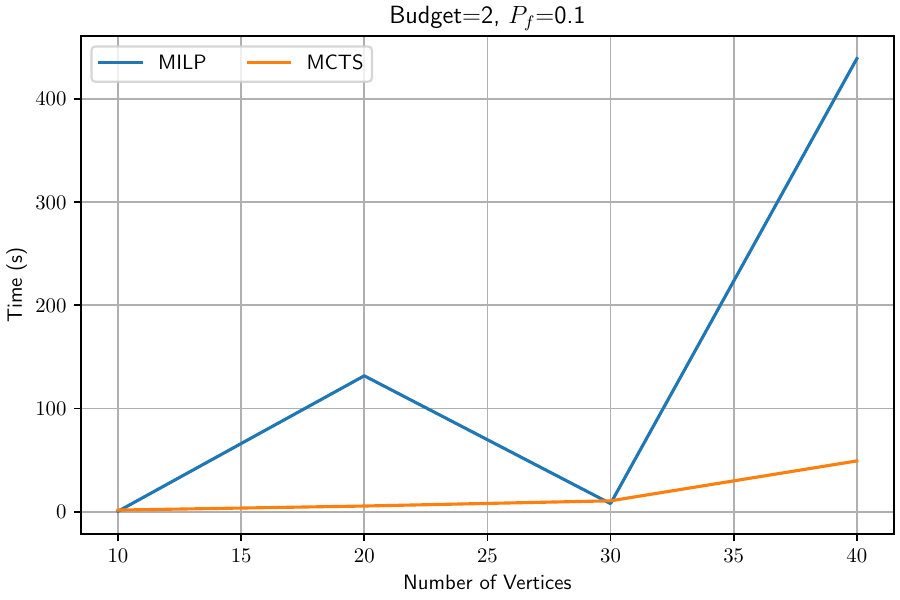}
	\caption{Average time spent to find the solution by the MILP algorithm (blue) and MCTS algoritm (orange) for 
		the case $B=2$. The top panel is for $P_f=0.05$ and the bottom one is for $P_f=0.1$.}
	\label{fig:time_milp_mcts2}
\end{figure}

\begin{figure}[htb]
	\centering
	\includegraphics[width=0.9\columnwidth]{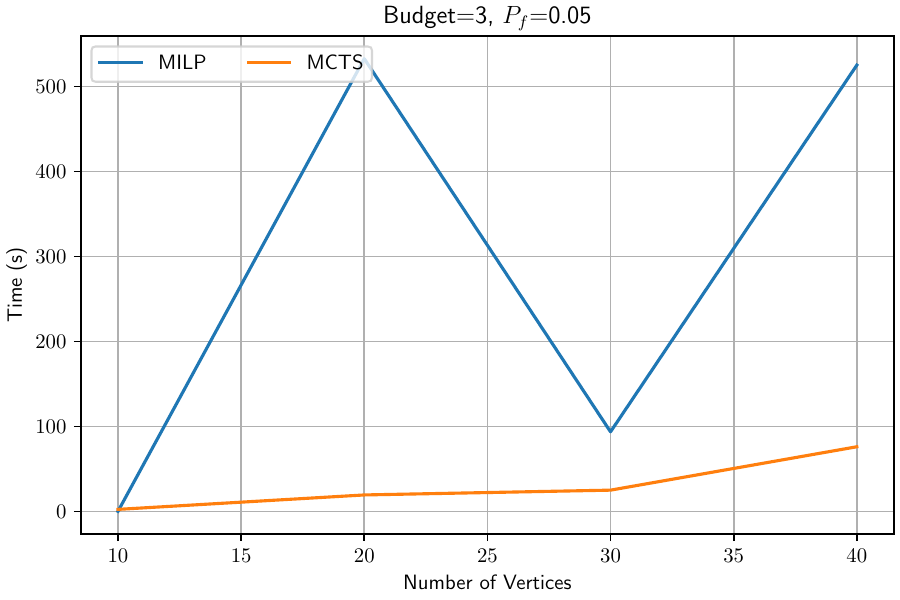}
	\includegraphics[width=0.9\columnwidth]{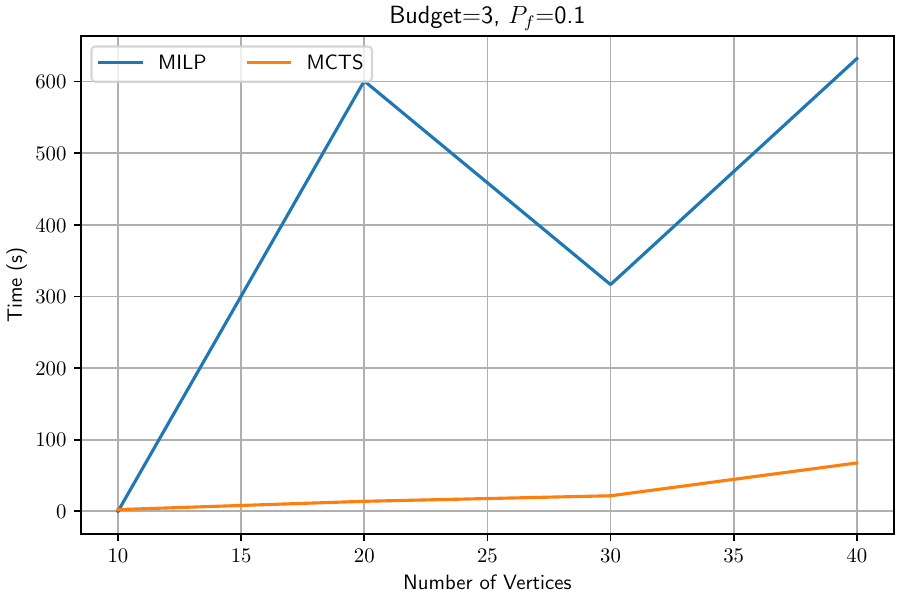}
	\caption{Average time spent to find the solution by the MILP algorithm (blue) and MCTS algoritm (orange) for 
		the case $B=3$. The top panel is for $P_f=0.05$ and the bottom one is for $P_f=0.1$.}
	\label{fig:time_milp_mcts3}
\end{figure}

In this case, we observe that for small problem instances (10 vertices) the time spent by the two 
algorithms is similar, with the MILP algorithm actually being even faster than the MCTS algorithm.
However, as the size of the problem increases, as expected, the MILP algorithm needs much more
time, with the exception of the case $B=2$, $n=30$, where evidently the structure of the graph
allows to quickly prune the search space. In all other cases, the gap is very large.

Overall these results suggest that for large problem sizes the proposed MCTS algorithm allows us to find high quality solutions
(albeit suboptimal) spending a fraction of the time and incurring  a lower failure rate. In a practical scenario,
one could actually implement both algorithms, and use MILP for small problem instances, and MCTS for large ones.

	Having assessed that the algorithm we propose favorably compares to the state of the art when
	solving instances we generated on our own, we conclude this experimental evaluation by further 
	comparing them on benchmark instances used by the traveling salesman problem (TSP) research 
	community and freely available online.\footnote{\url{ http://webhotel4.ruc.dk/~keld/research/LKH-3/}}
	The instances  we uses have a number of vertices ranging from 16 to 70 and we used the same
	noise distribution given in Eq.~\eqref{eq:noise} to  describe the edge cost noise.
	As TSP problem instances do not include vertex rewards, we assigned rewards using a uniform distribution
	over the range $[1,4]$. Figure \ref{fig:newgraphs} shows these additional problem instances, while Table \ref{tab:results}
	displays the results. Note that for the MILP approach in this case we had to increase the $Q$ value to 120,
	otherwise the failure probability exceeds the assigned $P_f$ threshold.
	In this case we observe that, as expected, as the number of vertices grows, the performance gap
	between the approximate MCTS approach and the MILP solution grows, too, but never drops below 0.72.
	At the same time, in all cases MCTS is always at least 50 times faster than the MILP solution.
	It is also worthwhile mentioning that for both algorithms one could perform a grid search to determine
	the optimal parameter values, but this has not been done. Overall, these additional benchmark
	problems show that the results we obtained with our own problem instances generalize also to
	other test cases.

	\begin{figure*}[htb]
		\centering
		\includegraphics[width = 0.3 \textwidth]{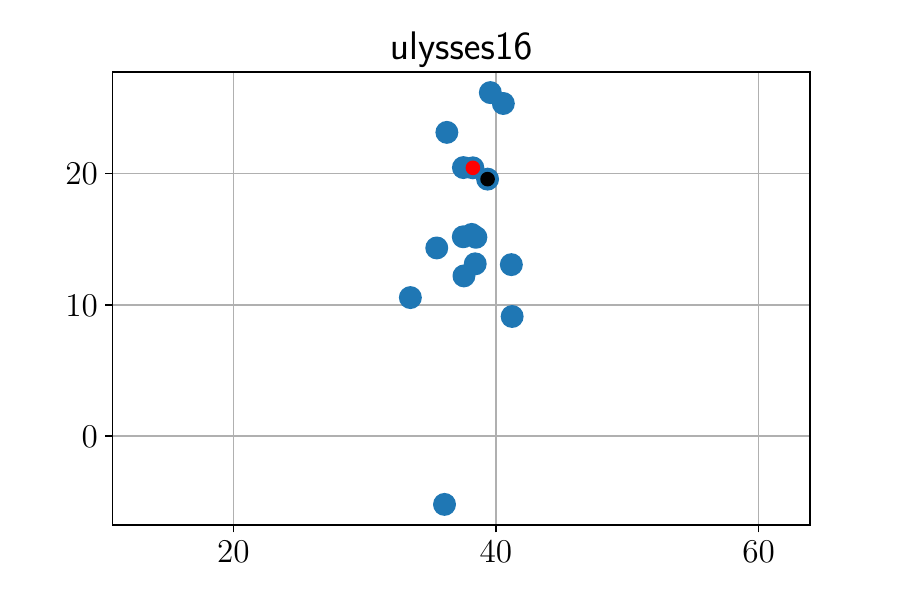}
		\includegraphics[width = 0.3 \textwidth]{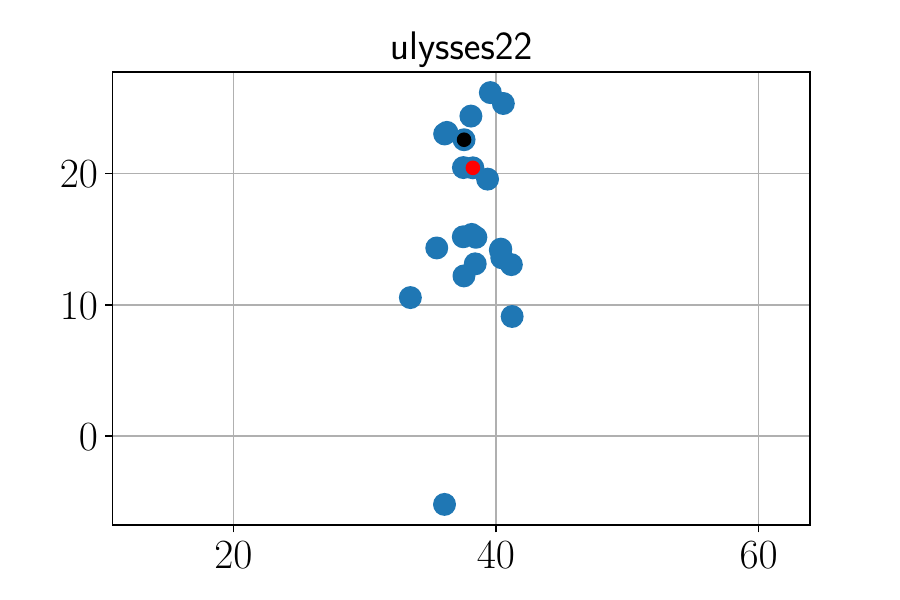}
		\includegraphics[width = 0.3 \textwidth]{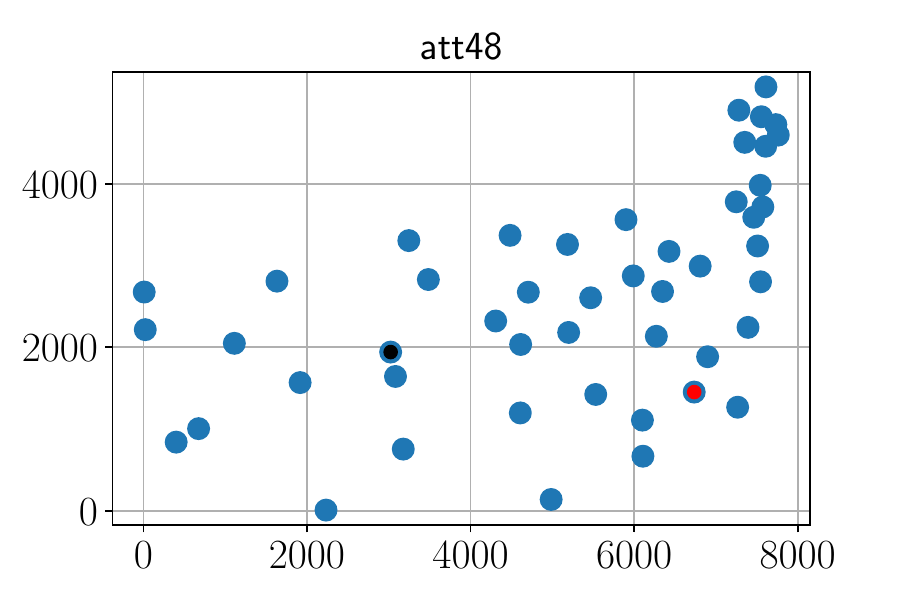}
		\includegraphics[width = 0.3 \textwidth]{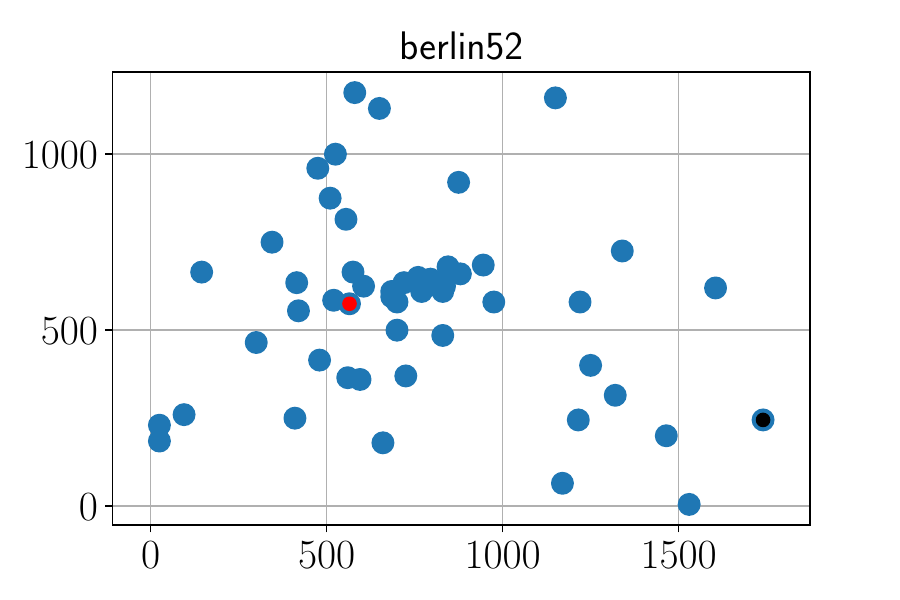}
		\includegraphics[width = 0.3 \textwidth]{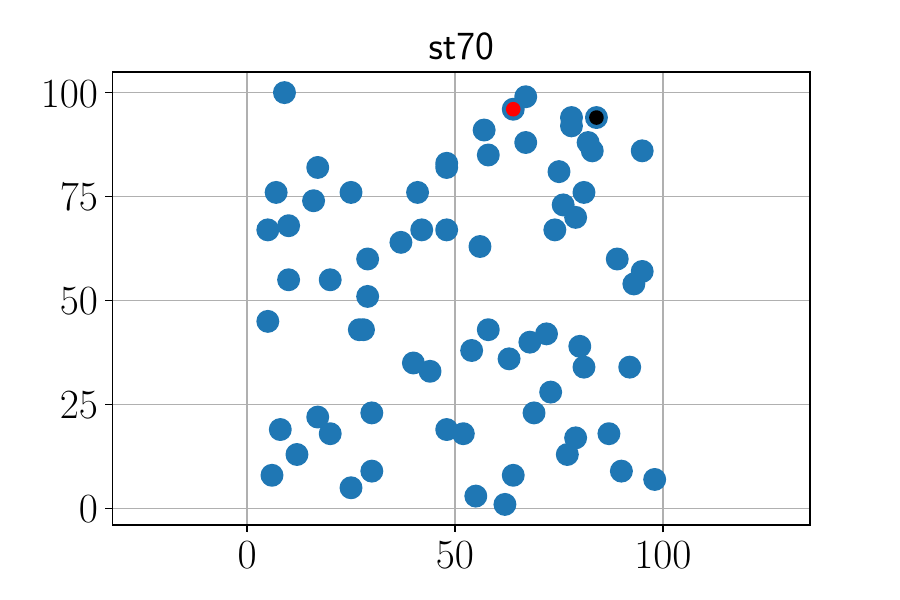}
		\caption{The five  additional benchmark problems borrowed from the TSP research community.
		The number of vertices varies from 16 to 70 (the numeric suffix in the benchmark name gives the number of vertices).}
		\label{fig:newgraphs}
	\end{figure*}

\begin{table*}[htb]
	\centering
		
\begin{tabular}{|c|c|c|c|c|c||c|c|c||c|c|}
	\cline{4-11}
	\multicolumn{3}{c|}{} & \multicolumn{3}{c||}{MCTS} & \multicolumn{3}{c||}{MILP} & \multicolumn{2}{c|}{Ratios}\\
	\hline
Test case	& Budget &  $P_f$ & Reward &$P_f$  & Time(s)  & Reward  & $P_f$ & Time(s)  & Reward &  Time   \\
	\hline
ulysses16 & 50 & 0.05 & 38.55 & 0.05 & 2.44 & 39.41 & 0.064 & 374.21 & 0.98  & 153.20\\
ulysses16 & 50 & 0.1 & 39.32 & 0.08 & 2.46 & 40.68 & 0.095 & 600.29 & 0.97  & 243.58\\
ulysses22 & 50 & 0.05 & 54.67 & 0.02 & 3.17 & 59.36 & 0.0465 & 600.71 & 0.92  & 189.09\\
ulysses22 & 50 & 0.1 & 54.63 & 0.05 & 3.08 & 59.86 & 0.078 & 600.59 & 0.91  & 194.74\\
att48 & 25000 & 0.05 & 104.14 & 0.08 & 7.67 & 134.36 & 0.0590 & 600.75 & 0.77  & 78.29\\
att48 & 25000 & 0.1 & 104.79 & 0.1 & 8.21 & 136.09 & 0.0815 & 600.74 & 0.77  & 73.17\\
berlin52 & 5000 & 0.05 & 109.28 & 0.06 & 8.01 & 141.74 & 0.0735 & 600.84 & 0.77  & 74.99\\
berlin52 & 5000 & 0.1 & 110.99 & 0.08 & 8.21 & 142.52 & 0.0875 & 669.16 & 0.78  & 81.42\\
st70 & 500 & 0.05 & 125.83 & 0.02 & 11.54 & 173.50 & 0.0745 & 629.03 & 0.72  & 54.50\\
st70 & 500 & 0.1 & 126.65 & 0.02 & 11.96 & 174.77 & 0.0785 & 622.38 & 0.72  & 52.00\\
	\hline
\end{tabular}
\caption{Comparison between the proposed MCTS approach and the MILP solution. For both algorithms we show
the collected reward, the experimental failure probability $P_f$ and the time spent to find the solution. The tenth column
shows the ratio between the reward obtained by MCTS and the reward obtained by MILP. In this case values closer to one are 
desirable. The eleventh column shows the ratio between the time spent by MILP and the time spent by MCTS. In this case, larger
values indicate more significant speedups. For the MCTS approach we consider averages over 100 runs, while for the MILP approach
we consider averages over 10 runs.}
\label{tab:results}
\end{table*}

\section{Conclusions and Future Work} 
\label{sec:conclusions}

In this paper, we have presented a new algorithm to solve the stochastic orienteering problem with chance
constraints. The main novelty is in using an MCTS approach to explore the space of possible policies and
a generative model for the travel times across the edges to prune away realizations that violate the 
assigned bound on the probability of failure. To the best of our knowledge, this approach is novel.
The proposed algorithm offers various advantages. By being an anytime algorithm, it can produce results
within a given pre-assigned computational time, while  previous methods will not produce any results
until the associated linear program is built and solved. In addition, we have shown that in a variety 
of test cases our method produces high quality solutions in a fraction of the time.

There are various venues for further research. First, the current rollout policy relies on a simple
greedy strategy. It would be interesting to explore whether different rollout policies (e.g., policies using
some of the heuristic methods formerly proposed for the orienteering problems) would produce better results (either
in terms of speed or collected rewards).

	From an application standpoint, it would also be interesting to consider the case of stochastic rewards,
	for both  cases when edge costs are deterministic or stochastic. With stochastic rewards
	and deterministic edge costs, one could consider a chance constraint on the minimum reward to collect.
	In this case, the proposed method would be mostly applicable 	``as is'' after having replaced the
	rollout method with one sampling collected reward values, rather than travel costs. The case with 
	both stochastic rewards and stochastic edge costs, appears to be more complex, as one could define
	separate chance constraints for both. In this case, one could extend the proposed method with rollouts sampling both
	stochastic quantities, but it would also be necessary to revise the problem formulation accounting for multiple constraints.
	These extensions will be considered in future iterations of this research. 

Finally, it will be interesting to investigate whether this approach
could be extended to study other planning problems with chance constraints where building a finite
state space is disadvantageous because of the need to discretize continuous dimensions, such as time.

\section*{Acknowledgments}
This paper extends and complements preliminary results presented in \cite{CarpinCASE2022}.
The author thanks Thomas C. Thayer for his contributions to the early stages of this research
and Roberta Marazzato for creating Figure 1.

	\bibliographystyle{IEEEtran}
\bibliography{../../../bibtex/orienteering.bib,../../../bibtex/carpin.bib,../../../bibtex/books.bib,../../../bibtex/planning.bib}

\end{document}